\documentclass[twoside]{article}

\usepackage[utf8]{inputenc}
\usepackage{multicol}
\usepackage{amsfonts}
\usepackage{amsmath}
\usepackage{amsthm}
\usepackage{amssymb}
\usepackage[font=small,labelfont=bf]{caption}\usepackage{xcolor}
\usepackage{graphicx}
\usepackage[colorlinks, hypertexnames=false]{hyperref}        
\usepackage{scrextend}
\usepackage{wrapfig}
\hypersetup{citecolor=blue}
\usepackage{booktabs}       
\usepackage{multirow}
\usepackage{authblk}
\usepackage{todonotes}
\usepackage{cite}
\usepackage{algorithm}
\usepackage{float}
\usepackage{stfloats}
\usepackage{threeparttable}
\usepackage{enumitem}
\usepackage[noend]{algpseudocode}
\usepackage{verbatim}
\algrenewcommand\alglinenumber[1]{\tiny #1:}
\usepackage[round]{natbib}

\usepackage[accepted]{aistats2022}
\newcommand{\B}{\mathcal{B}}

\newcommand{\dd}{\mathrm{d}}
\newcommand{\F}{\mathcal F}
\newcommand{\kl}{\mathrm{kl}}
\newcommand{\KL}{\mathrm{K\!L}}
\newcommand{\E}{\mathbb{E}}
\newcommand{\Err}{\mathcal{E}}

\newcommand{\FH}{\mathcal{F}^\Hh}
\newcommand{\Hh}{\mathcal{H}}
\newcommand{\indep}{\perp \!\!\! \perp}

\newcommand{\Ll}{\mathcal{L}}
\newcommand{\LO}{L_1}
\newcommand{\LT}{L_2}
\newcommand{\N}{\mathcal{N}}
\newcommand{\Pp}{\mathbb{P}}
\newcommand{\pp}{\mathfrak{p}}
\newcommand{\PR}{\mathcal{P}}
\newcommand{\Q}{\mathcal{Q}}
\newcommand{\R}{\mathbb{R}}

\newcommand{\V}{\mathbb{V}}
\newcommand{\X}{\mathcal{X}}
\newcommand{\Y}{\mathcal{Y}}

\newcommand{\emme}{\mathfrak{m}}
\newcommand{\esse}{\mathfrak{s}}

\DeclareMathOperator*{\argmax}{argmax}

\DeclareMathOperator*{\diag}{diag}
\DeclareMathOperator*{\erf}{erf}

\DeclareMathOperator*{\Pen}{\mathtt{Pen}}

\DeclareMathOperator*{\sign}{sign}

\newtheorem{lemma}{Lemma}
\newtheorem{prop}{Proposition}

\newtheorem{manpropin}{Proposition}
\newenvironment{manualprop}[1]{%
  \renewcommand\themanpropin{#1}%
  \manpropin
}{\endmanpropin}

\newcommand\blfootnote[1]{%
  \begingroup
  \renewcommand\thefootnote{}\footnote{#1}%
  \addtocounter{footnote}{-1}%
  \endgroup
}

\begin{document}
\runningauthor{Eugenio Clerico, George Deligiannidis, and Arnaud Doucet}

%

%

\twocolumn[

\aistatstitle{Conditionally Gaussian PAC-Bayes}

\aistatsauthor{Eugenio Clerico$^*$ \And George Deligiannidis \And  Arnaud Doucet\\\vspace{3pt}}

\aistatsaddress{Department of Statistics, University of Oxford} ]

\begin{abstract}
Recent studies have empirically investigated different methods to train stochastic neural networks on a classification task by optimising a PAC-Bayesian bound via stochastic gradient descent. Most of these procedures need to replace the misclassification error with a surrogate loss, leading to a mismatch between the optimisation objective and the actual generalisation bound. The present paper proposes a novel training algorithm that optimises the PAC-Bayesian bound, without relying on any surrogate loss. Empirical results show that this approach outperforms currently available PAC-Bayesian training methods.
\end{abstract}

\section{INTRODUCTION}
Understanding generalisation for neural networks is among the most challenging tasks for learning theorists \citep{allen-zhu-learning2019, understandingdeeplearning, explgen, thissues, kawaguchi2017generalization}. Only a few of the theoretical tools developed in the literature can produce non-vacuous bounds on the error rates of over-parametrised architectures, and PAC-Bayesian bounds have proven to be among the tightest in the context of supervised classification \citep{ambroladze, pacbayesmargins, McAllester2004PACBayesianSM}. Several recent works have focused on algorithms aiming to minimise a generalisation bound for stochastic classifiers by optimising a PAC-Bayesian objective via stochastic gradient descent; see e.g.\ \cite{alquier2016properties,dziugaite2017computing, perezortiz2021tighter, clerico2021wide, biggs2020differentiable, zhou2018nonvacuous, letarte2020dichotomize, perezortiz2021learning}. Most of these studies use a surrogate loss to avoid dealing with the zero-gradient of the misclassification loss. 
However, there are exceptions, such as \cite{biggs2020differentiable} and \cite{clerico2021wide}, which rely on the fact that an analytically tractable output distribution allows for an estimate of the misclassification error with a non-zero gradient with respect to the trainable parameters of the classifier.

\cite{clerico2021wide} treat the case of a stochastic network with a single hidden layer. They prove a Central Limit Theorem (CLT) ensuring the convergence of the output distribution to a multivariate Gaussian, whose mean and covariance can be evaluated explicitly in terms of the network deterministic hyper-parameters. However, this result cannot be straightforwardly extended to the multilayer case, as the nodes of the deeper layers are not independent and so the CLT might not apply. Moreover, even assuming that the output is Gaussian, the computational cost of this method is prohibitive for deep architectures.\blfootnote{$*$ Correspondence to: \textit{clerico@stats.ox.ac.uk}\\\vspace{-5pt}}

In \cite{biggs2020differentiable}, the focus is on a stochastic binary classifier whose output is of the form $\sign (w\cdot a)$, where $w$ is a Gaussian vector and $a$ is the output of the last hidden layer. The explicit form of the conditional expectation of the network's output (conditioned with respect to $a$) allows for a PAC-Bayesian training method applicable to arbitrarily deep networks. Nevertheless, this approach is only suitable for binary classification and cannot be easily extended to the multiclass case.

In the present work, we conjugate the two above ideas: in order to train the network with a method inspired by the Gaussian PAC-Bayesian approach from \cite{clerico2021wide}, we exploit the output's Gaussianity that can be obtained by conditioning on the previous layers, as in \cite{biggs2020differentiable}. This training procedure can be applied to a fairly general class of stochastic classifiers, overcoming some of the main limitations of the two aforementioned works, namely the single hidden layer and the binary classification setting. The main requirement for our method to be valid is that the parameters of the last linear layer are independent Gaussian random variables. Additionally, as we are not relying on any CLT result to obtain the Gaussianity, we do not need the network to be very wide for the algorithm to work. Consequently, the approach we propose can be computationally much cheaper than the one from \cite{clerico2021wide}.

We empirically validate our training algorithm on MNIST and CIFAR10 for a range of architectures, testing both data-dependent and data-free PAC-Bayesian priors. We compare our results to those from \cite{perezortiz2021tighter}, as, to our knowledge, these are currently the tightest empirical PAC-Bayesian bounds available on these datasets. Our novel approach outperforms their standard PAC-Bayesian training methods in all our experiments.
\section{BACKGROUND}
\subsection{PAC-Bayesian framework}
In a standard classification problem, to each instance $x\in\X\subseteq\R^p$ corresponds a true label $y = f(x) \in \Y=\{1,\dots,q\}$. A training set $S=(X_{k})_{k=1,\dots,m}$ is correctly labelled: for every $X_k \in S$ we have access to $Y_k=f(X_k)$. Each $X_k$ is an independent draw from a fixed probability measure $\Pp_X$ on $\X$, so that $S\sim\Pp_S = \Pp_X^{\otimes m}$. We consider a neural network, namely a parameterised function $F^\theta:\R^p\to\R^q$. For each input $x$, the network returns a prediction $\hat y$, defined as the largest output's node index:
$$\hat y = \hat f^\theta(x) =\textstyle{\argmax_{i \in \{1,\dots,q\}}} F^\theta_i(x)\,.$$
The goal is to train the net to make good predictions, exploiting the information in $S$ to tune the parameters.

Define the misclassification loss as
\begin{equation}\label{eq:01-loss}
\ell(\hat y, y) = \begin{cases} 0&\text{if $y=\hat y$}\,;\\ 1&\text{otherwise.}\end{cases}
\end{equation}
For a given configuration $\theta$ of the network parameters, we call empirical error the empirical mean of the misclassification loss on the training sample: $\Err_S(\theta) = \tfrac{1}{m}\sum_{x\in S}\ell(\hat f^\theta(x), f(x))$. This quantity can be explicitly evaluated, as we have access to the true labels on $S$. Therefore, it can be seen as an estimate for the true error $\Err_\Pp(\theta) = \E_X[\ell(\hat f^\theta(X), f(X))]=\Pp_X(\hat f^\theta(X)\neq f(X))$, which in general cannot be computed exactly.

The PAC-Bayesian bounds are upper bounds on the true error, holding with high probability on the choice of the training sample $S$; see e.g.\ \cite{mcall-pac-bayesian, mcallester, guedj2019primer, catoniPAC,alquier2021user}. A main feature of the PAC-Bayesian framework is that it requires the network to be stochastic, that is we are dealing with architectures whose parameters $\theta$ are random variables.

Let us fix $\PR$, a probability measure for the parameters $\theta$. We assume that $\PR$ is data-independent, in the sense that it has to be selected without accessing the information in the training sample $S$. In line with most PAC-Bayesian literature, we will refer to $\PR$ as the prior distribution. For a stochastic network, the training consists in efficiently modifying the distribution of $\theta$. This leads to a new distribution $\Q$ on the parameters, usually referred to as the posterior distribution. The main idea behind the PAC-Bayesian theory is that if the posterior $\Q$ is not ``too far'' from the prior $\PR$, then the network should not be prone to overfitting. The essential tool to measure this ``distance'' between the prior and posterior distributions is the Kullback--Leibler divergence, defined as 
$$\KL(\Q\|\PR) = \begin{cases}
\E_{\theta\sim \Q}\left[\log \tfrac{\dd \Q}{\dd \PR}(\theta)\right]&\text{if $\Q\ll\PR\,$};\\
+\infty&\text{otherwise.}
\end{cases}$$ 
The PAC-Bayesian bounds are upper bounds on the expected value of the true classification error $\Err_\Pp$ with respect to the posterior $\Q$. Two main ingredients constitute these bounds: the expected empirical error under $\Q$ and a complexity term, involving the divergence $\KL(\Q\|\PR)$. For simplicity, we will introduce the notations $\Err_\Pp(\Q) = \E_{\theta\sim\Q}[\Err_\Pp(\theta)]$ and $\Err_S(\Q) = \E_{\theta\sim\Q}[\Err_S(\theta)]$. The next proposition states some frequently used PAC-Bayes bounds \citep{langford_bounds, maurer, mcallester, Thiemann2017ASQ, perezortiz2021tighter}.

\begin{prop}\label{prop:Maurer_bound}
Fix $\delta\in(0,1)$, a data-independent prior $\PR$, and a training set $S=(X_k)_{k=1,\dots,m}$ drawn according to $\Pp_S$. Define
\begin{align}
    &\Pen = \tfrac{1}{m}\big(\KL(\Q\|\PR)+\log\tfrac{2\sqrt m}{\delta}\big)\,;\label{eq:defpen}\\
    &\,\kl^{-1}(u|c) = \sup\{v\in[0,1]:\kl(u\|v)\leq c\}\,,\label{eq:def_kl^-1}
\end{align}
where $\kl(u\|v)$ denotes the $\KL$ divergence between two Bernoulli distributions, with means $u$ and $v$ respectively.
Then, with probability at least $1-\delta$ on the random draw of the training set, for any posterior $\Q$ each of the following quantities upper bounds $\Err_\Pp(\Q)$:\footnote{For \eqref{eq:bound} we additionally assume that $S$ has size $m\geq 8$.\label{foot:mgeq8}}
\begin{subequations}\label{eq:bounds}
\begin{align}
    &\B_1= \kl^{-1}\left(\Err_S(\Q)|\Pen\right)\,;\label{eq:bound}\\
    &\B_2=\Err_S(\Q) + \sqrt{\Pen/2}\,;\label{eq:mcall}\\
    &\B_3=\big(\sqrt{\Err_S(\Q)+ \Pen/2} + \sqrt{\Pen/2}\big)^2\,;\label{eq:quad}\\
    &\B_4=\inf_{\lambda\in(0,1)}\tfrac{1}{1-\lambda/2}(\Err_S(\Q) + \Pen/\lambda)\,.
\label{eq:lbd}
\end{align}
\end{subequations}
\end{prop}
In the above proposition, the bound $\B_1$ is always the tightest. Moreover, all the above bounds are still valid if the empirical classification error $\Err_S$ is replaced by the empirical average of any loss function $\tilde\ell$ in $[0,1]$.

So far, we have assumed the prior $\PR$ to be data-independent. However, empirical evidence shows that using a data-dependent prior can lead to much tighter generalisation bounds, see e.g.\ \cite{parrado12a, perezortiz2021learning, pmlr-v130-karolina-dziugaite21a, Pacdiffpr, ambroladze}. Indeed, the actual requirement for the bounds \eqref{eq:bounds} to hold is that $\PR$ is independent of the sample $S$ used to evaluate $\Err_S(\Q)$. Hence, one can split the dataset $S$ into two disjoint sets, $S^{(1)}$ and $S^{(2)}$, use $S^{(1)}$ to train the prior, and obtain the data-dependent versions of the PAC-Bayesian bounds from Proposition  \ref{prop:Maurer_bound}, by redefining $\Pen=(\KL(\Q\|\PR_{S^{(1)}})+\log\frac{2\sqrt m_2}{\delta})/m_2$ and replacing $\Err_S(Q)$ with $\Err_{S^{(2)}}(\Q)$. For instance \eqref{eq:bound} becomes
\begin{equation}\label{eq:bound2}
    \Err_\Pp(\Q)\leq \kl^{-1}\left(\Err_{S^{(2)}}(\Q)\bigg|\frac{\KL(\Q\|\PR_{S^{(1)}})+\log\frac{2\sqrt m_2}{\delta}}{m_2}\right)\,,
\end{equation}
where $m_2\geq 8$ is the size of $S^{(2)}$.
\subsection{PAC-Bayesian training}
Ideally, one would like to implement the following procedure \citep{mcall-pac-bayesian}:
\begin{itemize}[noitemsep,topsep=0pt]
    \item Fix the PAC parameter $\delta\in(0,1)$ and a prior $\PR$ for the network stochastic parameters;
    \item Collect a sample $S$ of $m$ iid data points, according to $\Pp_S=\Pp_X^{\otimes m}$, and label it correctly;
    \item Compute an optimal posterior $\Q$ minimising a generalisation bound, such as \eqref{eq:bound};
    \item Implement a stochastic network whose random parameters have distribution $\Q$.
\end{itemize}
Unfortunately, in most realistic non-trivial scenarios, it can be extremely hard to compute and sample from an optimal posterior $\Q$ \citep{guedj2019primer}. A possible approach consists in using Markov chain Monte Carlo \citep{dala12, alquierbiau, guedjalquier}, sequential Monte Carlo or variational methods \citep{alquier2016properties}, in order to approximately sample from the Gibbs posterior, which can be shown to be the optimal $\Q$ when the PAC-Bayesian bound is linear in the empirical loss \citep{catoniPAC}. However, these methods can often be inefficient, especially in the case of deep architectures and large datasets.

An alternative approach relies on simplifying the problem by constraining $\PR$ and $\Q$ to belong to some simple distribution class. A common choice is to focus on the case of multivariate Gaussian distributions with diagonal covariance \citep{dziugaite2017computing, perezortiz2021tighter}: all the parameters are independent normal random variables. Conveniently, in this case the law of the random parameters can be easily expressed in terms of their means and standard deviations. These are deterministic trainable quantities that we will call hyper-parameters and denote by $\mathfrak{p}$. Furthermore, with this choice of $\PR$ and $\Q$, the $\KL$ divergence between prior and posterior takes a simple closed-form. Denoting as $\emme$ and $\esse$ (resp.\ $\widetilde \emme$ and $\widetilde \esse$) the means and standard deviations of the posterior (resp.\ prior), we have
\begin{align*}
    \begin{split}
    \KL(\Q\|\PR) = \frac{1}{2}\sum_k\frac{\esse_k^2-\widetilde\esse_k^2}{\widetilde\esse_k^2}+\frac{1}{2}\sum_k&\left(\frac{\emme_k-\widetilde\emme_k}{\widetilde\esse_k}\right)^2 \\&\;\;\;+\sum_k\log\frac{\widetilde\esse_k}{\esse_k}\,,
    \end{split}
\end{align*}
where the index $k$ runs over all the stochastic parameters of the networks. 

Now, the idea is to tune the hyper-parameters $\mathfrak{p}=(\emme,\esse)$ to minimise a PAC-Bayesian bound, such as \eqref{eq:bound}. A natural way to proceed is to perform a numerical optimisation via stochastic gradient descent, an approach originally proposed by \cite{germain09} and \cite{dziugaite2017computing}, and referred to as PAC-Bayes with BackProp by \cite{perezortiz2021tighter}. First, we fix a PAC-Bayesian bound as our optimisation objective. As previously mentioned, this will be an expression involving a complexity term and the empirical error ($\Pen$ and $\Err_S(\Q)$ respectively). We will hence denote it as $\B(\Err_S(\Q),\Pen)$. Generally, an explicit form for $\Err_S(\Q)$ is not available, but sampling from $\Q$ easily provides an unbiased estimate $\hat\Err_S(\Q)$ of this quantity. However, we cannot perform a gradient descent step on $\B(\hat\Err_S(\Q),\Pen)$. Indeed, $\hat\Err_S(\Q)$ has a null gradient almost everywhere, as it is the average over a finite set of realisations of the misclassification loss, which is constant almost everywhere \citep{perezortiz2021tighter}. In order to overcome this problem, it is common to use a surrogate loss function (usually a bounded version of the cross-entropy) instead of the misclassification loss; see e.g. \citep{dziugaite2017computing} and \citep{perezortiz2021tighter,perezortiz2021learning}. However, this creates a mismatch between the optimisation objective and the actual target bound.

It is worth noting that the zero-gradient problem is due to the particular form of the estimate $\hat\Err_S(\Q)$ and in general $\Err_S(\Q)$ has a non-zero gradient \citep{clerico2021wide}. Indeed, as it will be shown in Section \ref{sec:Cond-Gauss}, a different choice of estimator for $\Err_S(\Q)$ can allow training the network without the use of any surrogate loss.
\subsection{Stochastic network and notations}
Consider a stochastic classifier featuring several hidden layers and a final linear layer. We denote $H(x)$ the output of the last hidden layer when the input is $x$, $\phi$ the activation function (here applied component-wise), and $W$ and $B$ the weight and bias of the linear output layer. The output of the network will be
\begin{equation}\label{eq:output}
    F(x) = W\phi(H(x)) + B\,,
\end{equation}
where we wrote $F$ instead of $F^\theta$ to simplify the notation. Since the network is stochastic, $W$, $B$, and $H(x)$ are random quantities. We denote $\F^\mathcal{L}$ the $\sigma$-algebra generated by the last layer's stochasticity, and $\F^\mathcal{H}$ the one due to the hidden layers.

We will henceforth assume the following:
\begin{itemize}[noitemsep,topsep=0pt]
    \item $\F^\mathcal{L}\indep\F^\mathcal{H}$, that is the two $\sigma$-algebras are independent;
    \item $W$ and $B$ have independent normal components.
\end{itemize}
We can thus express the stochastic parameters of the last layer in terms of a set of deterministic trainable hyper-parameters $\emme$ and $\esse$:
\begin{align*}
    &W_{ij} = \zeta^W_{ij} \esse^W_{ij} + \emme^W_{ij}\,;&&B_i = \zeta^B_i \esse^B_i + \emme^B_i\,,
\end{align*}
where the $\zeta$ are all independent standard normal random variables $\mathcal N(0,1)$.

For the hidden layers, we do not require any strong assumption: essentially, we need to be able to sample a realisation $h(x)$ of $H(x)$, to evaluate the $\KL$ divergence between prior and posterior, and to differentiate both $\KL$ and $h(x)$ with respect to the trainable deterministic hyper-parameters. However, for the sake of simplicity, in the rest of this paper we will assume that all the parameters of the hidden layers have independent normal laws, as in \cite{dziugaite2017computing, clerico2021wide, perezortiz2021tighter}. All the architectures used for our experiments are indeed in this form. We refer to the supplementary material (Section \ref{sm:general}) for the extension of our results on more general architectures.
\section{COND-GAUSS ALGORITHM}\label{sec:Cond-Gauss}
We present here a training procedure to optimise a PAC-Bayesian generalisation bound without the need for a surrogate loss. The two main ideas are the following:
\begin{itemize}[noitemsep,topsep=0pt]
    \item An unbiased estimate of $\Err_S(\Q)$ and its gradient can be evaluated if the output of the network is Gaussian, as in \cite{clerico2021wide};
    \item If the parameters of the last layer are Gaussian, the output of the network is Gaussian as well when conditioned on the nodes of the last hidden layer, as pointed out by \cite{biggs2020differentiable}. 
\end{itemize}
\subsection{Gaussian output}\label{subsec:Gaussian output}
Fix an input $x$ and assume that the network's output $F(x)$ follows a multivariate normal distribution, with mean vector $M(x)$ and covariance matrix $Q(x)$. For our purposes, we can suppose that $Q(x)$ is diagonal, meaning that the components of the output are mutually independent (we refer to Section 4.1 in \cite{clerico2021wide} for the discussion of the general case). Let us denote $V(x)$ the diagonal of $Q(x)$, consisting of the output's variances, so that
$$\E_\Q[F_i(x)] = M_i(x)\,;\qquad\V_\Q[F_i(x)] = V_i(x)\,.$$
The stochastic prediction of our classifier is $\hat y = \hat f(x) = \argmax_{i\in\{1,\dots,q\}} F_i(x)$. In order to compute $\Err_S(\Q)$, for each input $x\in S$ we shall evaluate $\E_\Q[\ell(\hat f(x), f(x)]$. As $\ell$ is the misclassification loss \eqref{eq:01-loss}, this is simply the probability of making a mistake for the input $x$. Letting $y = f(x)$ and $\hat y=\hat f(x)$, we have
\begin{equation}\label{eq:F>F}
\E_\Q[\ell(\hat y, y)] = \Pp_\Q(\hat y \neq y) = \Pp_\Q\left(F_y(x)\leq \max_{i\neq y} F_i(x)\right)\,.
\end{equation}
In the case of binary classification, the above expression has a simple closed-form. Indeed, if we consider for instance the case $y=1$, we have
\begin{align*}
    \Pp_\Q(\hat y \neq 1) &= \Pp_\Q(F_2(x) - F_1(x) \geq 0)\\
    &= \Pp\left(\zeta\leq \frac{M_2(x) - M_1(x)}{\sqrt{V_1(x) + V_2(x)}}\right)\,,
\end{align*}
where $\zeta\sim\mathcal N(0,1)$. This can be expressed in terms of the error function $\erf$, as the cumulative distribution function of a standard normal is given by $\psi(u) = \Pp(\zeta\leq u) = \frac{1}{2}(1+\erf(u/\sqrt 2))$. Notice that the above expression no longer suffers from vanishing gradients, as $\psi'\neq 0$.

For multiple classes $(q>2)$, \eqref{eq:F>F} does not have a simple closed-form. However, we can easily find Monte Carlo estimators that also bring unbiased estimates for the gradient with respect to $M$ and $Q$.

\begin{prop}\label{prop:unbiassed}
Denote the cumulative distribution function of a standard normal random variable as $\psi:u\mapsto\frac{1}{2}(1+\erf(u/\sqrt 2))$. Fix $x$, let $y$ be its true label, and $\hat y$ the network's stochastic prediction. Define 
\begin{align*}
    &\LO = \psi\left(\max_{i\neq y}\frac{F_i(x) - M_y(x)}{\sqrt{V_y(x)}}\right)\,;\\
    &\LT = 1-\prod_{i\neq y}\psi\left(\frac{F_y(x)-M_i(x)}{\sqrt{ V_i(x)}}\right)\,,
\end{align*}
where $F(x) \sim\N(M(x), \diag(V(x)))$. 
Then 
\begin{align*}
    &\E_\Q[\LO] = \E_\Q[\LT] = \Pp_\Q(\hat y\neq y)\,,\\
    &\E_\Q[\nabla\LO] = \E_\Q[\nabla\LT] = \nabla\Pp_\Q(\hat y\neq y)\,,
\end{align*}
where the gradient is with respect to all the components of $M(x)$ and $V(x)$.\\
In particular, by sampling realisations of $L_1$ or $L_2$, we can get unbiased Monte Carlo estimators of the misclassification loss and its gradient. 
\end{prop}
\subsection{Conditional Gaussianity}
\begin{algorithm*}[bp]
\small
\caption{Cond-Gauss PAC-Bayesian training}
\label{alg:algo}
\begin{algorithmic}[1]
\Require
\Statex $\widetilde{\mathfrak{p}} = (\widetilde{\mathfrak{p}}^\mathcal{H},\widetilde{\mathfrak{p}}^\mathcal{L})$ \Comment{Initial hyper-parameters (defining the prior)}
\Statex $S$                       \Comment{Training set of size $\#S$}
\Statex $\delta\in(0,1)$         \Comment{PAC parameter}
\Statex $\eta, T$                 \Comment{Learning rate and number of epochs}
\Ensure
\Statex Optimal $\mathfrak{p}$ parameterizing the posterior
\Statex
\Procedure{Cond-Gauss}{}
    \State$\mathfrak{p}^\mathcal{H}\leftarrow\widetilde{\mathfrak{p}}^\mathcal{H}$
    \State$\mathfrak{p}^\mathcal{L} = (\emme, \esse)\leftarrow\widetilde{\mathfrak{p}}^\mathcal{L}$
    \For{$t\leftarrow 1:T$}
        \State Sample $\theta^\mathcal{H} \sim\Q^\mathcal{H}_{\mathfrak{p}^\mathcal{H}}$  \Comment{Sample the parameters of the hidden layers}
        \State $h = h(S, \theta^\mathcal{H})$     \Comment{Evaluate the last hidden layer's output for all $x\in S$}
        \State $M = M(h, \emme) = \emme^W\phi(h) + \emme^B$
        \Comment{Evaluate the conditional mean of the output}
        \State $V = V(h, \esse) = (\esse^W\phi(h))^2 + (\esse^B)^2$
        \Comment{Evaluate the conditional variance of the output}
        \State $\hat\Err_S(\Q_{\mathfrak{p}}) = \Err(M, V)$
        \Comment{Evaluate $\hat\Err_S(\Q_{\mathfrak{p}})$ from $M$ and $V$ as in Section \ref{subsec:Gaussian output}}
        \State $\hat\B = \B(\hat\Err_S(\Q_{\mathfrak{p}}), \Pen)$ \Comment{Evaluate the estimate $\hat\B$ of the PAC-Bayesian objective $\B$}
        \State $\mathfrak{p}\leftarrow\mathfrak{p}-\eta\nabla_{\mathfrak{p}}\hat\B$
        \Comment{Perform the gradient step}
    \EndFor
    \State \Return $\mathfrak{p}$
\EndProcedure
\end{algorithmic}
\end{algorithm*}
In practice, the output of a stochastic network is generally not Gaussian. However, we can overcome this issue by conditioning on the hidden layers, similarly to what was done by \cite{biggs2020differentiable}. 

Recall that the network's output is given by \eqref{eq:output}:
$$F = W\phi(H) + B\,,$$
where the explicit dependence of $H$ on $x$ is omitted to make the notations lighter. 
Conditioned on the stochasticity of the hidden layers $\F^\mathcal{H}$, $F$ follows a normal multivariate distribution, as
$$F = W\phi(H) + B \sim \mathcal N(M(H),Q(H))\,.$$

We can easily evaluate $M(H)$ and $Q(H)$ in terms of $\emme$ and $\esse$. We have
\begin{align*}
    M_i(H) = \E_\Q[F_i|\F^\mathcal{H}] &= \sum_j\E_\Q[W_{ij}]\phi(H_j) + \E_\Q[B_i]\\
    &= \sum_j\emme^W_{ij}\phi(H_j) + \emme^B_i
\end{align*}
and $Q_{ij}(H) = \delta_{ij} V_i(H)$, with
\begin{align*}
    V_i(H) = \V_\Q[F_i|\F^\mathcal{H}] &= \sum_j\V_\Q[W_{ij}]\phi(H_j)^2 + \V_\Q[B_i]\\
    &= \sum_j(\esse^W_{ij}\phi(H_j))^2 + (\esse^B_i)^2\,.
\end{align*}
Finally, we note that by iterated expectations 
$$\E_\Q[\ell(\hat f(x), f(x))] = \E_\Q[\E_\Q[\ell(\hat f(x), f(x))|\F^\mathcal{H}]]\,.$$
In particular, if we draw the hidden parameters and get a realisation $h$ of $H$, we obtain an unbiased estimate $\frac{1}{m}\sum_{x\in S}\E[\ell(\hat f(x), f(x))|H(x)=h(x)]$ of $\Err_S(\Q)$, where each term $\E[\ell(\hat f(x), f(x))|H(x)=h(x)]$ can be estimated with the methods from Section \ref{subsec:Gaussian output}, since $F(x)$ is a multivariate Gaussian when conditioned on $H(x)=h(x)$.
\subsection{Training algorithm}
We sketch here the Cond-Gauss training algorithm. First, we fix a PAC-Bayesian bound $\B$ as the optimisation objective. Then, we initialise the deterministic hyper-parameters of our network, and we select this configuration as the prior. 
Finally, we split our dataset into batches $S_1,\dots, S_K$. To train the network, we iterate over the batches and, similarly to what is done in most PAC-Bayesian training methods based on stochastic gradient descent, we sample the network's parameters at each batch iteration. However, we only perform this sampling for the hidden layers and not for the final linear layer. In this way, for each $x$ in the batch, we have a realisation $h(x)$ of the last hidden layer's output. Conditioned on $H=h$, the output is Gaussian and we can proceed as discussed earlier to get an estimate $\hat \Err_{S_k}(\Q, h)$ of $\Err_S(\Q)$. After that, we can obtain an estimate $\hat\B$ of the target bound $\B$, by replacing $\Err_S(\Q)$ with $\hat \Err_{S_k}(\Q, h)$. Finally, we compute the gradient of $\hat\B$ with respect to the trainable hyper-parameters, and we perform the gradient step.

If we want to use a data-dependent prior, we simply split the dataset into two subsets $S^{(1)}$ and $S^{(2)}$, and then use $S^{(1)}$ to learn $\PR$. For instance, we might train the prior using $\hat\Err_{S^{(1)}}(\Q)$ as optimisation objective or tuning only the prior's means by treating the network as if it was deterministic, similarly to what was done in \cite{perezortiz2021tighter}. Once the prior's training is complete, we perform the Cond-Gauss algorithm, replacing $S$ with $S^{(2)}$.

The training procedure is summarised in Algorithm \ref{alg:algo}, where,  for the sake of simplifying the notation, it is assumed that the whole training set forms a single batch. For convenience, we introduce the superscripts $^\mathcal{H}$ and $^\mathcal{L}$ to refer to the hidden layers and the last layer, respectively. Thus, we denote as $\theta = (\theta^\mathcal{H}, \theta^\mathcal{L})$ the random parameters of the network, where $\theta^\mathcal{H}$ are the parameters in the hidden layers, while $\theta^\mathcal{L}=(W, B)$ are those of the last layer. Similarly, $\mathfrak{p}^\mathcal{H}$ are the deterministic hyper-parameters relative to the hidden layers, whilst $\mathfrak{p}^\mathcal{L} = (\emme, \esse)$ are those of the last layer. We introduced the subscript $_\mathfrak{p}$ for the posterior $\Q$, to stress the fact that it is determined by the hyper-parameters, and we denoted by $\Q^\mathcal{H}$ the marginal posterior distribution for the hidden layers. Finally, the tilde notation represents the values at initialisation.

As a final remark, $\kl^{-1}$ is currently not implemented in most of the standard deep learning libraries. Yet, it can be easily computed numerically with few iterations of Newton's method, as in \cite{dziugaite2017computing}. Nevertheless, most of the empirical studies on PAC-Bayesian gradient descent optimisation (see e.g.\  \cite{dziugaite2017computing} and \cite{perezortiz2021tighter}), do not use as objective \eqref{eq:bound}, in order to avoid computing $\nabla\kl^{-1}$. However, since this gradient can be expressed as a function of $\kl^{-1}$ itself, we were able to optimise \eqref{eq:bound} in our experiments (see Section \ref{sm:kl-1} in the supplementary material for further details).
\subsection{Unbiasedness of the estimates}
One might wonder whether the estimates of $\B$ and its gradient are actually unbiased. Notably, this is indeed the case if the chosen PAC-Bayesian objective $\B$ is an affine function of the empirical error, as \eqref{eq:mcall} and \eqref{eq:lbd}.
\begin{prop}\label{prop:Bunbias}
Assume that $\B$ is locally Lipschitz in the hidden stochastic parameters $\theta^\Hh$, and that $\nabla_{\theta^\Hh}\B$ is polynomially bounded.\footnote{These are mild technical assumptions, verified in all the experimental settings considered in this paper.} If $\B(\Err_S(\Q), \Pen)$ is affine in $\Err_S(\Q)$, then we have $\E[\hat \B] = \B$ and $\E[\nabla\hat \B] = \nabla\B$, the gradient being with respect to the trainable hyper-parameters $\mathfrak{p}$. 
\end{prop}
Although this unbiasedness property does not hold for objectives not affine in $\Err_S(\Q)$, if $\hat\Err_S(\Q)$ concentrates enough around $\Err_S(\Q)$ we can linearise $\hat\B$ as
$$\hat\B \simeq \B+(\hat\Err_S(\Q)-\Err_S(\Q))\,\partial_\Err\B\,.$$
Then, both $\hat\B$ and $\nabla\hat\B$ are essentially almost unbiased estimates. Considering the good performance of our method in the experiments we ran, we conjecture that this is indeed what happens in practice with \eqref{eq:bound} and \eqref{eq:quad}. Figure \ref{fig:var_check} gives some empirical support to this hypothesis. We refer the reader to the supplementary material (Section \ref{sm:unbiased}) for additional discussion and empirical evidence on this subject.
\begin{figure}[tp]
\centering
      \includegraphics[width=\linewidth]{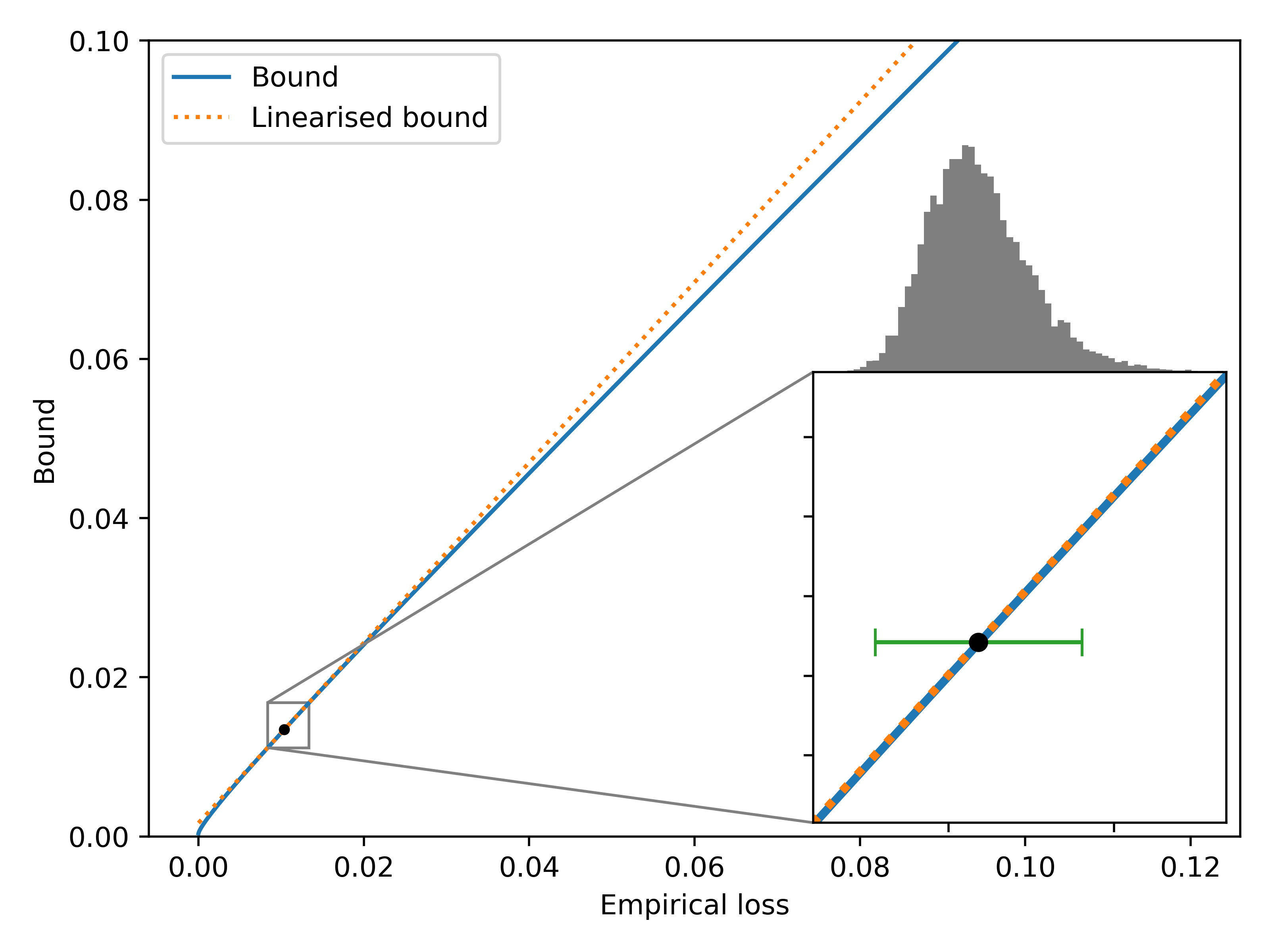}
      \vspace{-1em}
  \caption{Experimental evidence that the bound \eqref{eq:bound} is almost affine in the region where $\hat\Err_S(\Q)$ concentrates. The network used was the one achieving the best generalisation bound in our experiment on MNIST with \textit{data-dependent} priors. $10000$ realisations of $\hat \Err_S(\Q)$ were sampled. Their distribution is summarized by the histogram above the zoomed portion of the plot. The black dot is the bound for the average value found for $\hat\Err_S(\Q)$, while the green error bar has a total width of $4$ empirical standard deviations. In the region where $\hat\Err_S(\Q)$ concentrates, the bound and its linearised version almost coincide. Along the green error bar, the bound's slope has a relative variation of $\pm 0.8\%$.\\}
  \label{fig:var_check}
 \vspace{-1.5em}
\end{figure}
\subsection{Final evaluation of the bound}
In order to evaluate the final generalisation bound, we need the exact value of $\Err_S(\Q)$ once the training is complete. As this cannot be computed, we use an empirical upper bound, as done for instance in \citep{dziugaite2017computing}.

Let $\theta_1,\dots,\theta_N$ be $N$ independent realisations of the whole set of the network stochastic parameters, drawn according to $\Q$. An unbiased Monte Carlo estimator of $\Err_S(\Q)$ is simply given by 
$$\widetilde \Err_S(\Q) = \frac{1}{N}\sum_{n=1}^N \Err_S(\theta_n)\,.$$
As shown by \cite{langfordcaruana}, for fixed $\delta'\in(0,1)$, with probability at least $1-\delta'$ we have, 
\begin{equation*}\label{eq:emp_err_est}
\Err_S(\Q) \leq \kl^{-1}\left(\widetilde \Err_S(\Q)\big|\tfrac{1}{N}\log\tfrac{2}{\delta'}\right)\,,
\end{equation*}
where $\kl^{-1}$ is defined in \eqref{eq:def_kl^-1}. We conclude from Proposition \ref{prop:Maurer_bound} that, with probability higher than $1-(\delta+\delta')$, we have
\begin{align}\label{eq:final_bound}
\begin{split}
&\Err_\Pp(\Q)\\
&\leq \kl^{-1}\left(\kl^{-1}\left(\widetilde \Err_S(\Q)\big|\tfrac{1}{N}\log\tfrac{2}{\delta'}\right)\Big|\tfrac{\KL(\Q\|\PR)+\log\frac{2\sqrt m}{\delta}}{m}\right)\,,
\end{split}
\end{align}
as $\kl^{-1}$ is an increasing function of its first argument.
\section{NUMERICAL RESULTS}
We tested the Cond-Gauss algorithm empirically on the MNIST and the CIFAR10 datasets \citep{mnist, cifar}. In the literature, several works benchmark various PAC-Bayesian algorithms on these and other datasets \citep{dziugaite2017computing, Pacdiffpr, clerico2021wide, biggs2020differentiable, letarte2020dichotomize, perezortiz2021learning, perezortiz2021tighter}. To our knowledge, in the case of over-parameterised deep neural networks, the bounds from \cite{perezortiz2021tighter} are currently the tightest on both MNIST and CIFAR10. Thus, in order to assess our Cond-Gauss method by comparing their results with ours, we decided to mimic some of their multilayer convolutional architectures\footnote{The only difference between their architectures and ours is that we sometimes swapped the order between the application of the activation function and the max pooling. This fact was merely accidental, but we believe that it did not significantly affect our results.}, although our training schedules, as well as the prior's training procedures and the choice of initial variances, differed from theirs. All the generalisation bounds obtained with our training algorithm were tighter than those reported by \cite{perezortiz2021tighter}.

We illustrate below some of our main empirical results. All the final generalisation bounds are obtained from \eqref{eq:final_bound}, or its natural variant based on \eqref{eq:bound2} for data-dependent priors. We always use $\delta=0.025$ and $\delta'=0.01$ as in \cite{perezortiz2021tighter}, so that the final generalisation bounds hold with probability greater or equal to $0.965$. For all the bounds but those in Figure \ref{fig:MNIST_rnd}, we fixed $N=150000$ as in  \cite{perezortiz2021tighter}.

We refer to Section \ref{sm:experiments} in the supplementary material for the full results and the missing experimental details. The PyTorch code developed for this paper is available at \url{https://github.com/eclerico/CondGauss}. 
\subsection{MNIST}
For our experiments on MNIST, we only used the standard training dataset (60000 labelled examples) for the training procedure. We tested a 4-layer ReLU stochastic network, whose parameters were independent Gaussians. The architecture was composed of two convolutional layers followed by two linear layers.

We first experimented on data-free priors. We compared the performance of the standard PAC-Bayes with BackProp training algorithm (S), where the misclassification loss is replaced by a bounded version of the cross-entropy loss as in \cite{perezortiz2021tighter}, and the Cond-Gauss algorithm (G). We used the four training objectives from \eqref{eq:bounds}:
\begin{align*}
    &\mathtt{invKL}: &&\kl^{-1}(\Err_S(\Q)|{\Pen}_\kappa)\,;\\
    &\mathtt{McAll}: &&\Err_S(\Q) + \sqrt{{\Pen}_\kappa/2}\,;\\
    &\mathtt{quad}: &&(\sqrt{\Err_S(\Q)+ {\Pen}_\kappa/2} + \sqrt{{\Pen}_\kappa/2})^2\,;\\
    &\mathtt{lbd}: &&\tfrac{1}{1-\lambda/2}(\Err_S(\Q) +  {\Pen}_\kappa/\lambda)\,,
\end{align*}
where the $\KL$ penalty is defined as 
\begin{equation}\label{eq:def_Pen}
    {\Pen}_\kappa = \frac{\kappa}{m}\left(\KL(\Q\|\PR) + \log\frac{2\sqrt m}{\delta}\right)\,.
\end{equation}
The factor $\kappa$ in \eqref{eq:def_Pen} can increase or reduce the weight of the $\KL$ term during the training. For the last objective, $\mathtt{lbd}$, the parameter $\lambda$ takes values in $(0,1)$ and is optimised during training, similarly to what was done in \cite{perezortiz2021tighter}.\footnote{In our experiments, we initialised $\lambda$ at $0.5$ and then doubled the number of epochs, alternating one epoch of $\lambda$'s optimisation with one of optimisation for $\emme$ and $\esse$.}

The network was trained via SGD with momentum. During training, at the end of each epoch, we kept track of the bound \eqref{eq:bound}'s empirical value to pick the best epoch at the end of the training. 

In Figure \ref{fig:MNIST_rnd} we report the values of the bounds for different training settings with data-free priors on MNIST. As evaluating the true bound via \eqref{eq:final_bound} can be extremely time-consuming when $N=150000$, the values reported in the figure are obtained for $N=10000$. The Cond-Gauss algorithm always achieved the best performance. Note that some of the bounds in the figure are substantially tighter than the best value reported in \cite{perezortiz2021tighter}, namely $.2165$. Perhaps unexpectedly, this is sometimes the case even for the standard PAC-Bayes with BackProp algorithm, although it happened for training settings that were not tried therein, namely with the $\mathtt{invKL}$ objective or $\kappa=0.5$.

\begin{figure}[!t]
  \centering
      \includegraphics[width=\linewidth]{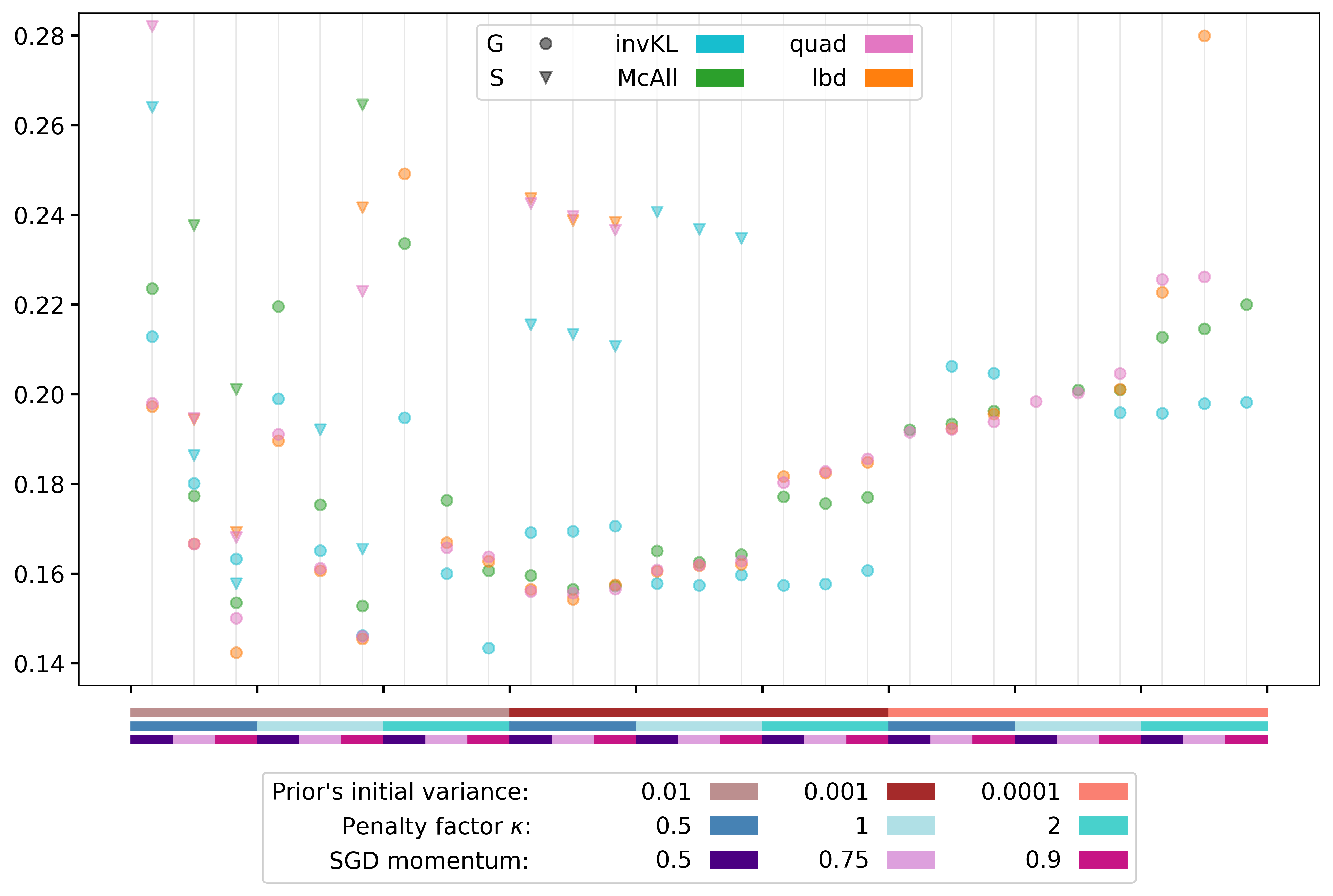}
      \vspace{-1em}
  \caption{Results for MNIST with random prior. Each dot is the PAC-Bayesian bound obtained via \eqref{eq:final_bound} with $N=10000$. The marker shape represents the training method (in the legend, `G' stands for our method, `S' for the standard one), 
 and the colour represents the training objective. Different columns indicate different momentum values, penalty factor $\kappa$, and initial variance for the prior. The initial prior's means were the same for all the different training possibilities. The values higher than $0.285$ are not reported.}\label{fig:MNIST_rnd}
 \vspace{-1.5em}
\end{figure}
\begin{table}[b]
\centering
\caption{PAC-Bayesian bounds for MNIST - data-free prior}
\label{tab:mnist_random}
\resizebox{\linewidth}{!}{\begin{tabular}{l|ccc|c}
  \toprule[1pt]
  method & emp err & test err & $\Pen$ \eqref{eq:defpen} & bound\\
  \midrule
    S $\mathtt{McAll}$ & .0670 & .0900\textsubscript{\textpm.0047} & .0320 & .1916\\
    S $\mathtt{lbd}$ & .0636 & .0623\textsubscript{\textpm.0013} & .0413 & .1606\\
    S $\mathtt{quad}$ & .0622 & .0577\textsubscript{\textpm.0031} & .0420 & .1594\\
    S $\mathtt{invKL}$ & .0438 & .0407\textsubscript{\textpm.0022} & .0560 & \textbf{.1495}\\
    \midrule
    G $\mathtt{McAll}$ & .0472 & .0435\textsubscript{\textpm.0024} & .0477 & .1446\\
    G $\mathtt{lbd}$ & .0279 & .0272\textsubscript{\textpm.0016} & .0669 & \textbf{.1348}\\
    G $\mathtt{quad}$ & .0399 & .0374\textsubscript{\textpm.0021} & .0518 & .1380\\
    G $\mathtt{invKL}$ & .0356 & .0340\textsubscript{\textpm.0019} & .0556 & .1355\\
    \bottomrule[1pt]
\end{tabular}}
\end{table}
In Table \ref{tab:mnist_random} we report the final generalisation bounds with $N=150000$, evaluated via \eqref{eq:final_bound}. For each method, we selected the training achieving the best bound in Figure \ref{fig:MNIST_rnd}. The Cond-Gauss procedure achieved better results than the standard algorithm with all the objectives. Quite surprisingly, the tightest bound was achieved by the $\mathtt{lbd}$ objective. The column `emp err' reports the empirical error on the training dataset, obtained when computing the final bounds. The test errors provided in the column `test err' are evaluated on the standard held-out test dataset of MNIST, by averaging over $1000$ realisations of the random network's parameter. We also report the empirical standard deviation of this estimate. Interestingly, the test error on the held-out dataset often resulted smaller than the empirical error on the training dataset. We do not have an explaination for this fact, which might be a mere coincidence and did not occur in most of the experiments with data-dependent priors. 

For the data-dependent priors, we used 50\% of the dataset to train $\PR$ and the remaining 50\% to train $\Q$. We always used the Cond-Gauss algorithm for both prior and posterior. All the posteriors were trained with the $\mathtt{invKL}$ objective and $\kappa=1$, whilst for the prior, we experimented with different objectives, penalty factors $\kappa$, and dropout values. The final best generalisation bound was $.0144$, about 7\% better than the tightest one from \cite{perezortiz2021tighter} for the same architecture, $.0155$. However, it is interesting to note that the role of the posterior's training seems to be quite marginal, as, in our experiments, the prior already achieved a quite low empirical error on the posterior's dataset, $.0108$, which could be improved only to $.0104$ by tuning the posterior. The results of the whole experiment can be found in Table \ref{tab:mnist_learnt} in the supplementary material.  
\subsection{CIFAR10}
As we had done for the MNIST dataset, for CIFAR10 we used only the standard training dataset (50000 labelled images) for the training procedure. We trained a 9-layer architecture (6 convolutional\ +\ 3 linear layers) and a 15-layer architecture (12 convolutional\ +\ 3 linear layers). We experimented with data-dependent priors only, training $\PR$ with 50\% of the data for the 9-layer classifier, and with both 50\% and 70\% of the data in the case of the 15-layer one. 
\begin{table*}[tp]
\centering
\caption{CIFAR10 - 9 layers - Prior learnt on 50\% of the dataset}
\label{tab:cifar_9_learnt}
\begin{threeparttable}
\resizebox{\textwidth}{!}{\begin{tabular}{lccc|ccc||l|cccc|c}
  \toprule[1pt]
  \multicolumn{7}{c||}{\textbf{Prior}} & \multicolumn{6}{c}{\textbf{Posterior}}\\
  \midrule
  tm\tnote{a} & do\tnote{b} & pf\tnote{c} & iv\tnote{d} & l1\tnote{e} & l2\tnote{f} & p\tnote{g} & tm\tnote{a} & l1\tnote{e} & l2\tnote{f} & t\tnote{h} & p\tnote{g} & b\tnote{i}\\
  \midrule
  G $\mathtt{invKL}$ & 0 & .01 & .001 & .0196 & .2233 & 3.778 & G $\mathtt{invKL}$ & .0196 & .2211 & .2251\textsubscript{\textpm.0021} & 4.696 & .2376\\
  G $\mathtt{invKL}$ & 0 & .005 & .001 & .1127 & .2797 & 3.778 & G $\mathtt{invKL}$ & .1126 & .2782 & .2814\textsubscript{\textpm.0019} & 4.319 & .2953\\
  G $\mathtt{invKL}$ & .1 & .01 & .001 & .0536 & .1930 & 3.778 & G $\mathtt{invKL}$ & .0536 & .1912 & .1952\textsubscript{\textpm.0020}& 4.484 & \textbf{.2066}\\
  G $\mathtt{invKL}$ & .1 & .005 & .001 & .0266 & .1930 & 3.778 & G $\mathtt{invKL}$ & .0266 & .1913 & .1933\textsubscript{\textpm.0019}& 4.520 & .2067\\
  \bottomrule[1pt]
\end{tabular}}
\begin{tablenotes}
\footnotesize{
\setlength{\columnsep}{0cm}
\setlength{\multicolsep}{0cm}
  \begin{multicols}{2}
    \item[a] tm: Training method.
    \item[b] do: Dropout probability for the prior's training.
    \item[c] pf: Penalty factor $\kappa$ for the prior's training objective. 
    \item[d] iv: Initial value of the prior's variances.
    \item[e] l1: Empirical error estimate on the prior dataset.
    \item[f] l2: Empirical error estimate on the posterior dataset.
    \item[g] t: Test error\textsubscript{\textpm standard deviation} (from 1000 realisations).
    \item[h] p: $\KL$ penalty $\Pen$ \eqref{eq:defpen} in $10^{-4}$ units.
    \item[i] b: Final PAC-Bayesian bound.
  \end{multicols}}
\end{tablenotes}
\end{threeparttable}
\end{table*}

The results for the 9-layer architecture are reported in Table \ref{tab:cifar_9_learnt}. Note that the best bound that we obtained in this setting was $.2066$, a result much tighter than the one reported by \cite{perezortiz2021tighter}, $.2901$. After some preliminary experiments, we chose to train both priors and posteriors via the Cond-Gauss algorithm with the $\mathtt{invKL}$ objective. We used a small factor $\kappa$ for the prior to avoid regularising too much, whilst $\kappa$ was $1$ for the posterior. We tried different values for the dropout and the factor $\kappa$ in the training of the prior, as reported in Table \ref{tab:cifar_9_learnt}. We trained via SGD with momentum for both prior and posterior. For $\PR$, we used a schedule much longer than the one usually chosen for the prior in the literature. Essentially, this is because we were not just training the means of the random parameters, but the variances as well. However, in this way we could already obtain for the priors competitive empirical errors on the posterior's dataset. Like with the MNIST dataset, the improvement due to the posterior's training was minimal.

For the 15-layer architecture, the full results and details are reported in Table \ref{tab:cifar_15_learnt} in the supplementary material. Quite interestingly, to train $\PR$, it was necessary to introduce an initial pre-training for the prior's means, as the Cond-Gauss algorithm alone could not significantly decrease the training objective. First, we initialised the means with an orthogonal initialisation, as suggested in \cite{huorthogonal}. Then we optimised them by training a deterministic network (with the same architecture) using the cross-entropy loss on the prior's dataset. Finally, via the Cond-Gauss algorithm, we completed the prior's training and proceeded with the posterior's tuning. The best final bounds obtained were $.1855$, with the prior learnt on 50\% of the dataset, and $.1595$, when 70\% of the dataset was used to train $\PR$. Again, these values are tighter than those from \cite{perezortiz2021tighter}.
\begin{table}[b]
\centering
\caption{Comparison of our PAC-Bayesian bounds with those from \cite{perezortiz2021tighter}}
\label{tab:bounds}
\begin{threeparttable}
\resizebox{\linewidth}{!}{\begin{tabular}{ccc|c|c}
  \toprule[1pt]
  dataset & architecture & prior & C-G & P-O\\
  \midrule
  MNIST & 4 layers & data-free & .1348 & .2165 \\
  MNIST & 4 layers & 50\% & .0144 & .0155 \\
  \midrule
  CIFAR10 & 9 layers & 50\% & .2066 & .2901 \\
  CIFAR10 & 15 layers & 50\% & .1855 & .1954 \\
  CIFAR10 & 15 layers & 70\% & .1595 & .1667 \\
  \bottomrule[1pt]
\end{tabular}}
\end{threeparttable}
\end{table}
\subsection{Summary}
To summarise our results, Table \ref{tab:bounds} compares our best PAC-Bayesian generalisation bounds with those from \cite{perezortiz2021tighter}. The column `C-G' features the best bounds we could obtain with the Cond-Gauss algorithm in our experiments. The figures in the column `P-O' are the tightest bounds reported in \cite{perezortiz2021tighter} for the same architectures and datasets. All the PAC-Bayesian generalisation bounds in the table hold with probability at least $0.965$ on the choice of the training dataset.
\section{CONCLUSION}
We have introduced the Cond-Gauss training algorithm, which allows the optimisation of PAC-Bayesian bounds without relying on the use of a surrogate loss. Taking an estimate of the actual target bound as the optimisation objective is a natural choice. As confirmed by our experiments on the MNIST and the CIFAR10 classification tasks, it also leads to tighter bounds than the current state-of-the-art bounds obtained via PAC-Bayes with BackProp. 
\subsubsection*{Acknowledgements}
The authors would like to acknowledge the use of the University of Oxford Advanced Research Computing (ARC) facility in carrying out this work.\footnote{\url{http://dx.doi.org/10.5281/zenodo.22558}} Eugenio Clerico is partly supported by the UK Engineering and Physical Sciences Research Council (EPSRC) through the grant EP/R513295/1 (DTP scheme). Arnaud Doucet is partly supported by the EPSRC grant EP/R034710/1. He also acknowledges the support of the UK Defence Science and Technology Laboratory (DSTL) and EPSRC under grant EP/R013616/1. This is part of the collaboration between US DOD, UK MOD, and UK EPSRC, under the Multidisciplinary University Research Initiative.
\newpage
\bibliography{bib}

\begin{thebibliography}{39}
\providecommand{\natexlab}[1]{#1}
\providecommand{\url}[1]{\texttt{#1}}
\expandafter\ifx\csname urlstyle\endcsname\relax
  \providecommand{\doi}[1]{doi: #1}\else
  \providecommand{\doi}{doi: \begingroup \urlstyle{rm}\Url}\fi

\bibitem[Allen-Zhu et~al.(2019)Allen-Zhu, Li, and
  Liang]{allen-zhu-learning2019}
Z.~Allen-Zhu, Y.~Li, and Y.~Liang.
\newblock Learning and generalization in overparameterized neural networks,
  going beyond two layers.
\newblock \emph{NeurIPS}, 2019.

\bibitem[Alquier(2021)]{alquier2021user}
P.~Alquier.
\newblock User-friendly introduction to {PAC-B}ayes bounds.
\newblock \emph{arXiv:2110.11216}, 2021.

\bibitem[Alquier and Biau(2013)]{alquierbiau}
P.~Alquier and G.~Biau.
\newblock Sparse single-index model.
\newblock \emph{Journal of Machine Learning Research}, 14, 2013.

\bibitem[Alquier et~al.(2016)Alquier, Ridgway, and
  Chopin]{alquier2016properties}
P.~Alquier, J.~Ridgway, and N.~Chopin.
\newblock On the properties of variational approximations of {G}ibbs
  posteriors.
\newblock \emph{Journal of Machine Learning Research}, 17, 2016.

\bibitem[Ambroladze et~al.(2007)Ambroladze, Parrado-{H}ern\'{a}ndez, and
  Shawe-{T}aylor]{ambroladze}
A.~Ambroladze, E.~Parrado-{H}ern\'{a}ndez, and J.~Shawe-{T}aylor.
\newblock Tighter {PAC-B}ayes bounds.
\newblock \emph{NeurIPS}, 2007.

\bibitem[Biggs and Guedj(2021)]{biggs2020differentiable}
F.~Biggs and B.~Guedj.
\newblock Differentiable {PAC}-{B}ayes objectives with partially aggregated
  neural networks.
\newblock \emph{Entropy}, 23\penalty0 (10), 2021.

\bibitem[Catoni(2007)]{catoniPAC}
O.~Catoni.
\newblock {PAC}-{B}ayesian supervised classification: The thermodynamics of
  statistical learning.
\newblock \emph{IMS Lecture Notes Monograph Series}, 2007.

\bibitem[Clerico et~al.(2021)Clerico, Deligiannidis, and
  Doucet]{clerico2021wide}
E.~Clerico, G.~Deligiannidis, and A.~Doucet.
\newblock Wide stochastic networks: {G}aussian limit and {PAC}-{B}ayesian
  training.
\newblock \emph{arXiv:2106.09798}, 2021.

\bibitem[Dalalyan and Tsybakov(2012)]{dala12}
A.~S. Dalalyan and A.~B. Tsybakov.
\newblock Sparse regression learning by aggregation and {L}angevin
  {M}onte-{C}arlo.
\newblock \emph{Journal of Computer and System Sciences}, 78\penalty0 (5),
  2012.

\bibitem[Deng(2012)]{mnist}
L.~Deng.
\newblock The {MNIST} database of handwritten digit images for machine learning
  research.
\newblock \emph{IEEE Signal Processing Magazine}, 29\penalty0 (6), 2012.

\bibitem[Dziugaite and Roy(2017)]{dziugaite2017computing}
G.~K. Dziugaite and D.~M. Roy.
\newblock Computing nonvacuous generalization bounds for deep (stochastic)
  neural networks with many more parameters than training data.
\newblock \emph{UAI}, 2017.

\bibitem[Dziugaite and Roy(2018)]{Pacdiffpr}
G.~K. Dziugaite and D.~M. Roy.
\newblock Data-dependent {PAC-B}ayes priors via differential privacy.
\newblock \emph{NeurIPS}, 2018.

\bibitem[Dziugaite et~al.(2021)Dziugaite, Hsu, Gharbieh, Arpino, and
  Roy]{pmlr-v130-karolina-dziugaite21a}
G.~K. Dziugaite, K.~Hsu, W.~Gharbieh, G.~Arpino, and D.~M. Roy.
\newblock On the role of data in {PAC-B}ayes.
\newblock \emph{AISTATS}, 2021.

\bibitem[Germain et~al.(2009)Germain, Lacasse, Laviolette, and
  Marchand]{germain09}
P.~Germain, A.~Lacasse, F.~Laviolette, and M.~Marchand.
\newblock {PAC-B}ayesian learning of linear classifiers.
\newblock \emph{ICML}, 2009.

\bibitem[Guedj(2019)]{guedj2019primer}
B.~Guedj.
\newblock A primer on {PAC}-{B}ayesian learning.
\newblock \emph{Proceedings of the Second Congress of the French Mathematical
  Society}, 2019.

\bibitem[Guedj and Alquier(2013)]{guedjalquier}
Benjamin Guedj and Pierre Alquier.
\newblock {PAC-Bayesian estimation and prediction in sparse additive models}.
\newblock \emph{Electronic Journal of Statistics}, 7, 2013.

\bibitem[Hu et~al.(2020)Hu, Xiao, and Pennington]{huorthogonal}
W.~Hu, L.~Xiao, and J.~Pennington.
\newblock Provable benefit of orthogonal initialization in optimizing deep
  linear networks.
\newblock \emph{ICLR}, 2020.

\bibitem[Kawaguchi et~al.(2017)Kawaguchi, Kaelbling, and
  Bengio]{kawaguchi2017generalization}
K.~Kawaguchi, L.~P. Kaelbling, and Y.~Bengio.
\newblock Generalization in deep learning.
\newblock \emph{arXiv:1710.05468}, 2017.

\bibitem[Khaled and Richt{\'a}rik(2020)]{khaled2020better}
A.~Khaled and P.~Richt{\'a}rik.
\newblock Better theory for {SGD} in the nonconvex world.
\newblock \emph{arXiv:2002.03329}, 2020.

\bibitem[Krizhevsky(2009)]{cifar}
A.~Krizhevsky.
\newblock Learning multiple layers of features from tiny images.
\newblock \emph{MSc Thesis University of Toronto}, 2009.

\bibitem[Langford and Caruana(2002)]{langfordcaruana}
J.~Langford and R.~Caruana.
\newblock ({N}ot) bounding the true error.
\newblock \emph{NeurIPS}, 2002.

\bibitem[Langford and Seeger(2001)]{langford_bounds}
J.~Langford and M.~Seeger.
\newblock Bounds for averaging classifiers.
\newblock \emph{CMU technical report}, 2001.

\bibitem[Langford and Shawe-Taylor(2003)]{pacbayesmargins}
J.~Langford and J.~Shawe-Taylor.
\newblock {PAC}-{B}ayes \& margins.
\newblock \emph{NeurIPS}, 2003.

\bibitem[Letarte et~al.(2019)Letarte, Germain, Guedj, and
  Laviolette]{letarte2020dichotomize}
G.~Letarte, P.~Germain, B.~Guedj, and F.~Laviolette.
\newblock Dichotomize and generalize: {PAC}-{B}ayesian binary activated deep
  neural networks.
\newblock \emph{NeurIPS}, 2019.

\bibitem[Maurer(2004)]{maurer}
A.~Maurer.
\newblock A note on the {PAC} {B}ayesian theorem.
\newblock \emph{arXiv:0411099}, 2004.

\bibitem[McAllester(1998)]{mcall-pac-bayesian}
D.~A. McAllester.
\newblock Some {PAC}-{B}ayesian theorems.
\newblock \emph{COLT}, 1998.

\bibitem[McAllester(1999)]{mcallester}
D.~A. McAllester.
\newblock {PAC}-{B}ayesian model averaging.
\newblock \emph{COLT}, 1999.

\bibitem[McAllester(2004)]{McAllester2004PACBayesianSM}
D.~A. McAllester.
\newblock {PAC-B}ayesian stochastic model selection.
\newblock \emph{Machine Learning}, 51, 2004.

\bibitem[Neyshabur et~al.(2017)Neyshabur, Bhojanapalli, McAllester, and
  Srebro]{explgen}
B.~Neyshabur, S.~Bhojanapalli, D.~McAllester, and N.~Srebro.
\newblock Exploring generalization in deep learning.
\newblock \emph{NeurIPS}, 2017.

\bibitem[Parrado-{H}ern{{\'a}}ndez et~al.(2012)Parrado-{H}ern{{\'a}}ndez,
  Ambroladze, Shawe-{T}aylor, and Sun]{parrado12a}
E.~Parrado-{H}ern{{\'a}}ndez, A.~Ambroladze, J.~Shawe-{T}aylor, and S.~Sun.
\newblock {PAC-B}ayes bounds with data dependent priors.
\newblock \emph{Journal of Machine Learning Research}, 13, 2012.

\bibitem[Poggio et~al.(2020)Poggio, Banburski, and Liao]{thissues}
T.~Poggio, A.~Banburski, and Q.~Liao.
\newblock Theoretical issues in deep networks.
\newblock \emph{PNAS}, 117\penalty0 (48), 2020.

\bibitem[Price(1958)]{price}
R.~Price.
\newblock A useful theorem for nonlinear devices having gaussian inputs.
\newblock \emph{IRE Transactions on Information Theory}, 4\penalty0 (2), 1958.

\bibitem[Pérez-Ortiz et~al.(2021{\natexlab{a}})Pérez-Ortiz, Risvaplata,
  Shawe-Taylor, and Szepesvári]{perezortiz2021tighter}
M.~Pérez-Ortiz, O.~Risvaplata, J.~Shawe-Taylor, and C.~Szepesvári.
\newblock Tighter risk certificates for neural networks.
\newblock \emph{Journal of Machine Learning Research}, 22, 2021{\natexlab{a}}.

\bibitem[Pérez-Ortiz et~al.(2021{\natexlab{b}})Pérez-Ortiz, Rivasplata,
  Guedj, Gleeson, Zhang, Shawe-Taylor, Bober, and
  Kittler]{perezortiz2021learning}
M.~Pérez-Ortiz, O.~Rivasplata, B.~Guedj, M.~Gleeson, J.~Zhang,
  J.~Shawe-Taylor, M.~Bober, and J.~Kittler.
\newblock Learning {PAC}-{B}ayes priors for probabilistic neural networks.
\newblock \emph{arXiv:2109.10304}, 2021{\natexlab{b}}.

\bibitem[Stein(1981)]{stein}
C.~M. Stein.
\newblock Estimation of the mean of a multivariate normal distribution.
\newblock \emph{The Annals of Statistics}, 9\penalty0 (6), 1981.

\bibitem[Tadić and Doucet(2017)]{tadic2017}
V.~B. Tadić and A.~Doucet.
\newblock Asymptotic bias of stochastic gradient search.
\newblock \emph{The Annals of Applied Probability}, 27\penalty0 (6), 2017.

\bibitem[Thiemann et~al.(2017)Thiemann, Igel, Wintenberger, and
  Seldin]{Thiemann2017ASQ}
N.~Thiemann, C.~Igel, O.~Wintenberger, and Y.~Seldin.
\newblock A strongly quasiconvex {PAC-B}ayesian bound.
\newblock \emph{ALT}, 2017.

\bibitem[Zhang et~al.(2021)Zhang, Bengio, Hardt, Recht, and
  Vinyals]{understandingdeeplearning}
C.~Zhang, S.~Bengio, M.~Hardt, B.~Recht, and O.~Vinyals.
\newblock Understanding deep learning (still) requires rethinking
  generalization.
\newblock \emph{Commun. ACM}, 64\penalty0 (3), 2021.

\bibitem[Zhou et~al.(2019)Zhou, Veitch, Austern, Adams, and
  Orbanz]{zhou2018nonvacuous}
W.~Zhou, V.~Veitch, M.~Austern, R.~P. Adams, and P.~Orbanz.
\newblock Non-vacuous generalization bounds at the imagenet scale: a
  {PAC}-{B}ayesian compression approach.
\newblock \emph{ICLR}, 2019.

\end{thebibliography}
\newpage
\onecolumn
\setcounter{equation}{0}
\setcounter{lemma}{0}
\setcounter{prop}{0}
\setcounter{section}{0}
\setcounter{theorem}{0}
\setcounter{table}{0}
\setcounter{figure}{0}
\section*{Supplementary material}
\renewcommand\theequation{SM\arabic{equation}}
\renewcommand{\thelemma}{SM\arabic{lemma}}
\renewcommand{\theprop}{SM\arabic{prop}}
\renewcommand{\thesection}{SM\arabic{section}}
\renewcommand{\thetheorem}{SM\arabic{theorem}}
\renewcommand{\thetable}{SM\arabic{table}}
\renewcommand{\thefigure}{SM\arabic{figure}}
\section{PROOFS}\label{sm:proofs}
\begin{manualprop}{\ref{prop:unbiassed}}
Denote the cumulative distribution function (CDF) of a standard normal as $\psi:u\mapsto\frac{1}{2}(1+\erf(u/\sqrt 2))$. Fix a pair $x, y$ and let 
\begin{align*}
    \LO = &\psi\left(\max_{i\neq y}\frac{F_i(x) - M_y(x)}{\sqrt{V_y(x)}}\right)\,,\\
    \LT = &1-\prod_{i\neq y}\psi\left(\frac{F_y(x)-M_i(x)}{\sqrt{ V_i(x)}}\right)\,,
\end{align*}
where $F(x) \sim\N(M(x), \diag(V(x)))$. 
Then 
\begin{align}
    &\E[\LO] = \E[\LT] = \Pp(\hat y\neq y)\,,\label{eq:propunb1}\\
    &\E[\nabla\LO] = \E[\nabla\LT] = \nabla\Pp(\hat y\neq y)\,,\label{eq:propunb2}
\end{align}
where the gradient is with respect to all the components of $M(x)$ and $V(x)$.
\end{manualprop}
\begin{proof}
We start by showing that $\E[\LO]=\Pp(\hat y\neq y)$. We have
$$\Pp(\hat y\neq y) = \Pp\left(F_y(x)< \max_{i\neq y} F_i(x)\right) = \E\left[\Pp\left(F_y(x)< \max_{i\neq y} F_i(x)\bigg|\{F_i(x)\}_{i\neq y}\right)\right] = \E[\LO]\,.$$
For $\LT$ again we first use conditioning w.r.t. $F_y(x)$
$$\Pp(\hat y\neq y) = \E\left[\Pp\left(F_y(x)< \max_{i\neq y} F_i(x)\bigg|F_y(x)\right)\right] = 1-\E\left[\Pp\left(F_y(x)\geq\max_{i\neq y} F_i(x)\bigg|F_y(x)\right)\right]\,.$$
As the events $\{F_y(x)\geq F_i(x)|F_y(x)\}_{i\neq y}$ are independent, we can write
$$\Pp(\hat y\neq y) = 1-\E\left[\prod_{i\neq y}\Pp\left(F_i(x)\leq F_y(x)\bigg| F_y(x)\right)\right] = \E[\LT]\,,$$
and so \eqref{eq:propunb1} is proved.

Now, to show \eqref{eq:propunb2}, we need to prove that it is possible to swap expectation and differentiation for both $\LO$ and $\LT$. For $\LT$ everything is straightforward, as it is a smooth function of $M$ and $V$ (as all the components of $V$ are assumed to be strictly positive) and its gradient can be easily bounded (uniformly in some neighbourhood of $(M_i(x), V_i(x))_{i\neq y}$) by a function of $F_y(x)$ with finite expectation. Hence we can apply Leibniz integral rule. For $\LO$, this is the case only for $\partial_{M_y}$ and $\partial_{V_y}$, as $\max_{i\neq y}\frac{F_i(x) - M_y(x)}{\sqrt{V_y(x)}} = \frac{\max{i\neq y}\{F_i(x)\} -M_y(x)}{\sqrt{ V_y(x)}}$ is smooth in $M_y$ and $V_y$, and its gradient can be easily bounded (uniformly in some neighbourhood of $(M_y(x), V_y(x))$) by a function of $(F_i(x))_{i\neq y}$ with finite expectation. However, for any $j\neq y$, the integrand is not everywhere differentiable wrt $M_j$ and $V_j$. Yet, we can still swap expectation and differentiation using Proposition \ref{prop:Gaussder}, detailed below.
\end{proof}

The two results that follow are well known in the literature, and restated here for convenience. For completeness we give a proof for both of them. Denote as $\rho_{m, s}$ the density of a normal random variable with mean $m$ and standard deviation $s$. For convenience we let $\rho=\rho_{0, 1}$. All integrals $\int$ are over $\R$. 

The next proposition is essentially a reformulation of Price's theorem \citep{price}.
\begin{prop}\label{prop:Gaussder}
Let $Z\sim\N(0,1)$ and $X=sZ+m$. Let $g:\R\to\R$ be a locally Lipschitz function with a polynomially bounded derivative. Then
$$\nabla_{m, s}\E_{X\sim\N(m, s^2)}[g(X)] = \E_{Z\sim\N(0, 1)}[\nabla_{m, s} g(sZ+m)]\,.$$
\end{prop}
\begin{proof}
Recall that $\partial_m \rho_{m, s}(x) = \frac{x-m}{s}\rho_{m,s}(x)$ and $\partial_s\rho_{m,s}(x) = \frac{(x-m)^2 - s^2}{s^3}\rho_{m,s}(x)$. Let $z=sx +m$, then $\rho_{m,s}(x)\dd x = \rho(z)  \dd z$. Note that by the local Lipschitzianity $g'$ is defined almost everywhere. Since it is polynomially bounded, the expectation $\E[\nabla_{m, s} g(sZ+m)]$ makes sense. Note moreover that $g$ is polynomially bounded as $g'$ is. 

We start by proving the equality for the $m$-derivative. We have
$$\partial_m \E[g(X)] = \partial_m\int\rho_{m, s}(x)g(x)\dd x = \int(\partial_m\rho_{m, s}(x))g(x)\dd x\,,$$
by Leibniz integration rule, as $\rho_{m, s}$ is smooth in its arguments and the continuity and polynomial boundedness of $g$ ensure that 
$\int x\to\partial_m\rho_{m, s}(x))g(x)\,\dd x$ is well defined and finite.
Now, we have
$$\int(\partial_m\rho_{m, s}(x))g(x)\dd x = \int\frac{x-m}{s^2}\rho_{m,s}(x)g(x)\dd x = \int\frac{z}{s}\rho(z)g(sz+m)\dd z\,.$$
From Lemma \ref{lem:gaussintparts} below, we get
$$\int\frac{z}{s}\rho(z)g(sz+m)\dd z = \int\frac{1}{s}\rho(z)sg'(sz + m)\dd z = \int\rho(z)g'(sz +m)\dd z\,.$$
Now, as $g'(sz + m) = \partial_m g(sz + m)$ we conclude that
$$\partial_m \E[g(X)] = \E[\partial_m g(sZ+m)]\,.$$
For the $s$-derivative, the proof is essentially analogous. Proceeding as above, we have
$$\partial_s \E[g(X)] = \int(\partial_s\rho_{m, s}(x))g(x)\dd x = \int\frac{(x-m)^2-s^2}{s^3}\rho_{m,s}(x)g(x)\dd x = \int\frac{z^2-1}{s}\rho(z)g(sz+m)\dd z\,.$$
Again from Lemma \ref{lem:gaussintparts} we find that
$$\int\frac{z^2-1}{s}\rho(z)g(sz+m)\dd z = \int \rho(z)zg'(sz+m)\dd z\,.$$
We conclude that 
$$\partial_s\E[g(x)] = \E[\partial_s g(sz + m)]\,,$$
since $\partial_s g(sz+m) = z g'(sz+m)$.
\end{proof}
The next lemma states Stein's identity \citep{stein} and a straightforward corollary.
\begin{lemma}\label{lem:gaussintparts}
Let $Z\sim\N(0,1)$, and $g:\R\to\R$ a locally Lipschitz function with a polynomially bounded derivative. Then
\begin{align*}
    &\E[Zg(Z)] = \E[g'(Z)]\,,\\
    &\E[(Z^2-1)g(Z)] = \E[Zg'(Z)]\,.
\end{align*}
\end{lemma}
\begin{proof}
The first equality, known as Stein's identity, is established using integration by parts:
\begin{align*}
    0 = \int(\rho(z)g(z))'\dd z = \int\rho'(z)g(z)\dd z + \int\rho(z)g'(z)\dd z = &-\int z\rho(z)g(z) + \int\rho(z)g'(z)\dd z\,,
\end{align*}
where we used that $g'$ exists almost everywhere as $g$ is locally Lipschitz, and that both $g$ and $g'$ are polynomially bounded, so all integral are finite and well defined. Now take $h(z) = z g(z)$. Then we have $h'(z) = zg'(z)+g(z)$ and so 
$$\E[Z^2g(Z)] = \E[Zh(Z)] = \E[h'(Z)] = \E[Zg'(Z)] + \E[g(Z)]\,,$$
which is the second equality. 
\end{proof}
\begin{manualprop}{\ref{prop:Bunbias}}
Assume that $\B$ is locally Lipschitz in the hidden stochastic parameters $\theta^\Hh$, and that $\nabla_{\theta^\Hh}\B$ is polynomially bounded. If $\B(\Err_S(\Q), \Pen)$ is an affine function of $\Err_S(\Q)$, then we have $\E[\hat \B] = \B$ and $\E[\nabla\hat \B] = \nabla\B$, the gradient being with respect to the trainable hyper-parameters $\mathfrak{p}$. 
\end{manualprop}
\begin{proof}
By linearity it is sufficient to show that $\E[\hat\Err_S(\Q)] = \Err_S(\Q)$ and $\E[\nabla\hat\Err_S(\Q)] = \nabla\Err_S(\Q)$. Note that, following the discussion of Section \ref{subsec:Gaussian output}, we can write $\hat\Err_S(\Q) = \sum_{x\in S}\hat\Err_x$ where
$$\hat\Err_x = E(M(x, \theta^\Hh, \pp^\Ll), V(x, \theta^\Hh, \pp^\Ll), \xi)\,,$$
for some suitable function $E$. If we are dealing with binary classification the variable $\xi$ can be omitted, otherwise it represents the random draws needed to obtain the estimate $L_1$ or $L_2$ (defined in Proposition \ref{prop:unbiassed}).

Define $\Err_x = \E[\hat\Err_x]$, the expectation being over $\xi$ and $\theta^\Hh$. By Proposition \ref{prop:unbiassed} (if we are dealing with multiclass classification, otherwise by definition) we get that $\Err_S(\Q)=\sum_{x\in S}\Err_x$. Consequently we have
$$\Err_S(\Q) = \sum_{x\in S} \E[\hat\Err_x] = \E[\hat\Err_S(\Q)]\,.$$ 
Now, to show the unbiasedness of the gradient, it is enough to show that for all $x\in S$ 
$$\nabla_\pp\Err_x = \E[\nabla_\pp\hat\Err_x]\,.$$
First, again by Proposition \ref{prop:unbiassed} we can write
$$\E[\nabla_\pp\hat\Err_x] = \E[\nabla_\pp\E[\hat\Err_x|\FH]] = \E[\tfrac{\partial(M, V)}{\partial\pp}\E[\nabla_{M,V}\hat\Err_x|\FH]] = \E[\tfrac{\partial(M, V)}{\partial\pp}\nabla_{M,V}\E[\hat\Err_x|\FH]] = \E[\nabla_\pp\E[\hat\Err_x|\FH]]\,.$$
Now, $\E[\hat\Err_x|\FH]$ is the probability that a component of a Gaussian vector with mean $M$ and covariance $\diag(V)$ is smaller than the maximum of the other components (cf.\ Section \ref{subsec:Gaussian output}). This is a smooth function of $M$ and $V$, which in turn are smooth functions of the last layer's hyper-parameters $\pp^\Ll$. As a consequence we can write
$$\nabla_{\pp^\Ll}\Err_x = \E[\nabla_{\pp^\Ll}\E[\hat\Err_x|\FH]] = \E[\nabla_{\pp^\Ll}\hat\Err_x]\,.$$
As for the hidden hyper-parameters, since we are assuming that all the hidden stochastic parameters are independent Gaussian random variables, we can apply Proposition \ref{prop:Gaussder}, which brings
$$\nabla_{\pp^\Hh}\Err_x = \E[\nabla_{\pp^\Hh}\E[\hat\Err_x|\FH]] = \E[\nabla_{\pp^\Hh}\hat\Err_x]\,,$$
thus concluding our proof.
\end{proof}
\section{A NOTE ON UNBIASEDNESS}\label{sm:unbiased}
The previous results state that the gradient estimates used in the Cond-Gauss algorithm are unbiased, as long as the bound is affine in the empirical error. Under suitable regularity conditions, this ensures that stochastic gradient descent algorithms converge to a stationary point of the objective \citep{khaled2020better}. However, among the four bounds \eqref{eq:bounds} that we used in our experiments, only \eqref{eq:mcall} and \eqref{eq:lbd} are actually affine. We argue here that in most cases of interest $\hat\Err_S(\Q)$ is  concentrated enough that the bounds \eqref{eq:bound} and \eqref{eq:quad} are approximately affine in the empirical error. In the following, we detail this heuristic idea and then give some empirical evidence on MNIST in the case of \eqref{eq:bound}. This almost affine behaviour ensures that the gradient used by our stochastic optimisation procedure is almost unbiased, and hence we can expect the algorithm to converge to a point close to a stationary point of the objective \citep{tadic2017}.

Consider a generic bound $\B = B(\Err_S(\Q))$, where $B$ might be a non-affine function. Our estimate is of the form $\hat\B = B(\hat\Err_S(\Q))$. We can now consider a linearised version $\bar B$ of $B$, defined as 
$$\bar B(\Err) = B(\Err_S(\Q)) + (\Err-\Err_S(\Q))B'(\Err_S(\Q))\,.$$
Clearly, in a sufficiently small neighborhood of $\Err_S(\Q)$, we can expect $B$ and $\bar B$ to almost coincide. In particular, if the law of $\hat\Err_S(\Q)$ concentrates around $\Err_S(\Q)$, we can expect that with high probability 
$$B(\hat\Err_S(\Q))\simeq \bar B(\Err_S(\Q))\,.$$
As $\bar B$ is affine, we can apply Proposition \ref{prop:Bunbias} and get
\begin{align*}
    &\E[\hat \B] = \E[B(\hat\Err_S(\Q))] \simeq \E[\bar B(\hat\Err_S(\Q))] = \bar B(\Err_S(\Q)) = B(\Err_S(\Q)) = \B\,,\\
    &\E[\nabla_\pp\hat \B] = \E[\nabla_\pp B(\hat\Err_S(\Q))] \simeq \E[\nabla_\pp \bar B(\hat\Err_S(\Q))] = \nabla_\pp \bar B(\Err_S(\Q)) = \nabla_\pp B(\Err_S(\Q)) = \nabla_\pp\B\,.\\
\end{align*}
To empirically justify the above, we consider the bound \eqref{eq:bound}, which was used for most of our experiments. Figure \ref{fig:var_check_app} and Figure \ref{fig:var_check_rnd} show that indeed $\hat\Err_S(\Q)$ is sufficiently concentrated around its mean to see the bound as an affine function of the empirical error. Figure \ref{fig:var_check_app} reports the data from the network achieving the best bound in our experiments with \textit{data-dependent} priors on MNIST. On the other hand, among the networks trained with the $\mathtt{invKL}$ objectives on MNIST with \textit{data-free} priors, the one achieving the tightest bound was used for Figure \ref{fig:var_check_rnd}. In both figures, the histogram represents the distribution of $10000$ realisations of $\hat\Err_S(\Q)$. It is clear that in both cases the bound is essentially affine in the empirical loss, in the region where $\hat\Err_S(\Q)$ concentrates (zoomed portion of the plot).

Similar observations hold when the objective is derived from \eqref{eq:quad}.
\begin{figure}[!th]
\centering
\begin{minipage}{.48\textwidth}
  \centering
      \includegraphics[width=\linewidth]{var_check.png}
      \vspace{-1em}
  \caption{(Same as Figure \ref{fig:var_check} from the main text.) Experimental evidence, from a network trained with a \textit{data-dependent} prior on MNIST, that the bound \eqref{eq:bound} is almost affine in the region where $\hat\Err_S(\Q)$ concentrates. The network used was the one achieving the best generalisation bound in our experiment on MNIST with \textit{data-dependent} priors. $10000$ realisations of $\hat \Err_S(\Q)$ were sampled. Their distribution is summarised by the histogram above the zoomed portion of the plot. The black dot is the bound for the average value found for $\hat\Err_S(\Q)$, while the green error bar has a total width of $4$ empirical standard deviations. In the region where $\hat\Err_S(\Q)$ concentrates, the bound and its linearised version almost coincide. Along the green error bar, the bound's slope has a relative variation of $\pm 0.8\%$.\\}
  \label{fig:var_check_app}
 \vspace{-1.5em}
\end{minipage}%
\hfill
\begin{minipage}{.48\textwidth}
  \centering
      \includegraphics[width=\linewidth]{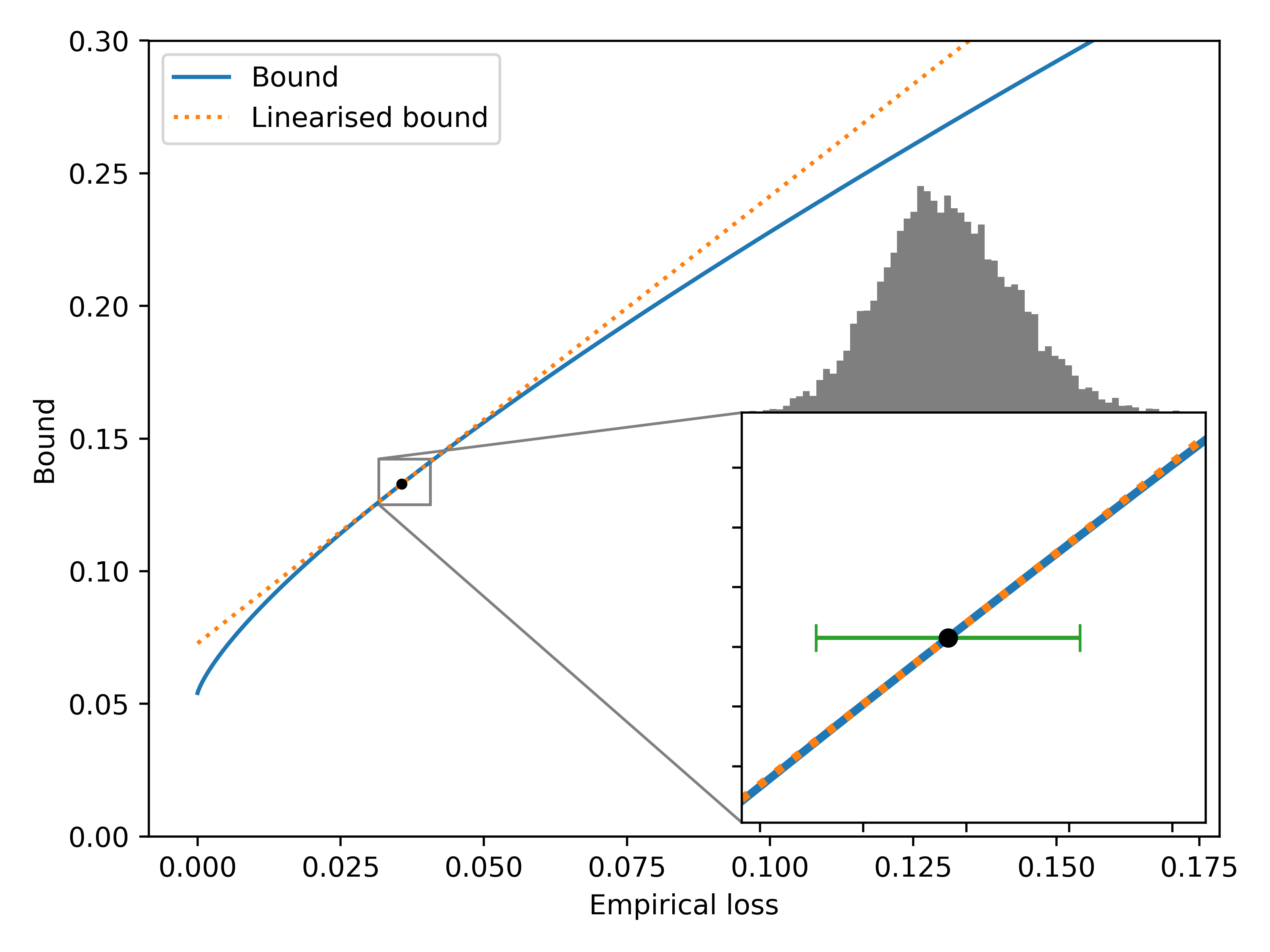}
      \vspace{-1em}
  \caption{Experimental evidence, from a network trained with a \textit{data-free} prior on MNIST, that the bound \eqref{eq:bound} is almost affine in the region where $\hat\Err_S(\Q)$ concentrates. Among the networks trained with the $\mathtt{invKL}$ objectives on MNIST with \textit{data-free} priors, the one achieving the tightest bound was used in this experiment. $10000$ realisations of $\hat \Err_S(\Q)$ were sampled. Their distribution is summarised by the histogram above the zoomed portion of the plot. The black dot is the bound for the average value found for $\hat\Err_S(\Q)$, while the green error bar has a total width of $4$ empirical standard deviations. In the region where $\hat\Err_S(\Q)$ concentrates, the bound and its linearised version almost coincide. Along the green error bar, the bound's slope has a relative variation of $\pm 2\%$.\\}
  \label{fig:var_check_rnd}
 \vspace{-1.5em}
\end{minipage}%
\end{figure}
\section{PAC-BAYESIAN TRAINING FOR GENERAL ARCHITECTURES}\label{sm:general}
In the main text we focused on the case of a network whose stochastic parameters are all Gaussian. This is not a necessary condition for the Cond-Gauss algorithm. What we need is actually to be able to express the $\KL$ between prior and posterior as a differentiable expression of the hyper-parameters, and to evaluate the gradient (wrt the hyper-parameters) of a single empirical loss's realisation. We can satisfy this last requirement if we are able to rewrite the stochastic parameters as a (differentiable) function $\Theta$ of the hyper-parameters $\pp$ and of some random variable $\tau$ (independent of $\pp$) such that $\Theta(\pp, \tau)$ has the same law of $\theta$, namely $\Q_\pp$. In short, for any measurable function $\varphi$,
$$\E_{\theta\sim\Q_\pp}[\varphi(\theta)] = \E_\tau[\varphi(\Theta(\pp, \tau))]\,.$$
In particular, to sample a realisation $\hat \varphi$ of $\varphi(\theta)$ we can sample a realisation $\hat\tau$ of $\tau$ and then define
$$\hat\varphi = \varphi(\Theta(\pp, \hat\tau))\,.$$
As long as $\varphi\circ\Theta$ is differentiable in $\pp$, we can evaluate the gradient of $\hat\varphi$ wrt $\pp$. 

For the Cond-Gauss algorithm to be implementable, we require that there exists a $\pp$-differentiable reparametrisation $\Theta$ for the hidden parameters $\theta^\mathcal{H}$. Clearly, this is the case if $\theta^\mathcal{H}$ is a Gaussian vector with independent components. Indeed, if we denote by $\emme^\mathcal H$ and $\esse^\mathcal H$ the vectors of means and standard deviations, we have
$$\theta^\mathcal H = \emme^\mathcal H+\esse^\mathcal H\odot\tau\,,$$
where $\tau$ is a vector with independent standard normal components and $\odot$ denotes the component-wise product. This is what was used for the networks in our experiments. 

\section{NUMERICAL EVALUATION OF \texorpdfstring{$\kl^{-1}$}{kl^-1} AND ITS GRADIENT}\label{sm:kl-1}
When the training objective is $\mathtt{invKL}$, it is necessary to evaluate $\kl^{-1}$ and its gradient, in order to implement the Cond-Gauss algorithm. Many of the most popular deep learning libraries, such as PyTorch and TensorFlow, do not provide an implementation for $\kl^{-1}$. However, as pointed out by \cite{dziugaite2017computing}, a fast numerical evaluation can be done via a few iterations of Newton's method. This is what we used in our code.

We show here that the gradient of $\kl^{-1}$ can be expressed as a function of $\kl^{-1}$, so that the implementation of the latter allows the evaluation of the former. Recall that 
$$\kl(u\|v) = u \log \frac{u}{v} + (1-u)\log\frac{1-u}{1-v}\,.$$
For $u>0$, the mapping $v\mapsto \kl(u\|v)$ is not injective. However if we restrict its domain to $\{(u,v)\in[0,1]^2:v\geq u\}$, then we find a bijective map, whose inverse coincides with $c\mapsto\kl^{-1}(u|c)$ (with the definition \eqref{eq:def_kl^-1} for $\kl^{-1}$). It follows immediately that
$$\partial_c \kl^{-1}(u|c) = \frac{1}{\partial_v \kl(u\|v)}\bigg|_{v=\kl^{-1}(u|c)} = \left(\frac{1-u}{1-v}-\frac{u}{v}\right)^{-1}\bigg|_{v=\kl^{-1}(u|c)}\,.$$
To find an expression for $\partial_u \kl^{-1}(u|c)$ we can proceed as follow. Let $\kl^{-1}(u|c) = v$ and $\kl^{-1}(u+\varepsilon|c) = v+ \varepsilon'$, with $\varepsilon' = \varepsilon\partial_u\kl^{-1}(u|c)+o(\varepsilon)$. This means that $\kl(u+\varepsilon\|v+\varepsilon') = \kl(u\|v)$, so that $\varepsilon\partial_u\kl(u\|v) + \varepsilon'\partial_v\kl(u\|v) = o(\varepsilon)$. Taking $\varepsilon\to 0$ we find 
$$\partial_u \kl^{-1}(u|c) = -\frac{\partial_u\kl(u\|v)}{\partial_v \kl(u\|v)}\bigg|_{v=\kl^{-1}(u|c)} = \left(\log\frac{1-u}{1-v}-\log\frac{u}{v}\right)\bigg/\left(\frac{1-u}{1-v}-\frac{u}{v}\right)\bigg|_{v=\kl^{-1}(u|c)}\,.$$
\section{ADDITIONAL EXPERIMENTAL DETAILS AND RESULTS}\label{sm:experiments}
In this section we give additional details about our experiments. The PyTorch code written for this paper is available at \url{https://github.com/eclerico/CondGauss}. In all our experiments we used the average of $100$ independent estimates of $L^1$ (defined in Proposition \ref{prop:unbiassed}) to evaluate the empirical error. To keep the standard deviations $\sigma$ positive during the training, we trained the parameters$\rho$ defined by $\sigma=|\rho|^{3/2}$. We found empirically that this transformation allowed for a much faster training compared to the usual exponential choices \citep{dziugaite2017computing, perezortiz2021tighter}.
\subsection{MNIST}
For our experiments on MNIST, we only used the standard training dataset, which consists of 60000 labelled examples. We ran our experiments on a 4-layer ReLU stochastic network, whose parameters were independent Gaussians with trainable means and variances. The architecture used was the following:
$$x\mapsto y = L_2\circ\phi\circ L_1\circ \phi\circ f\circ C_2\circ \phi \circ C_1(x)\,,$$
with
\begin{itemize}[noitemsep,topsep=0pt]
    \item $C_1$: convolutional layer; channels: IN 1, OUT 32; kernel: (3, 3); stride: (1, 1);
    \item $C_2$: convolutional layer; channels: IN 32, OUT 64; kernel: (3, 3); stride: (1, 1);
    \item $L_1$: linear layer; dimensions: IN 9216, OUT 128;
    \item $L_2$: linear layer; dimensions: IN 128, OUT 10;
    \item $f$: max pool (kernel size = 2) \& flatten;
    \item $\phi$: ReLU activation component-wise.
\end{itemize}
All convolutional and linear layers were with bias.
\subsubsection{Data-free priors}
We first experimented on data-free priors, whose means were initialised via the Pytorch default initialisation. We tried different values for the initial prior's variances: $.01$, $.001$, and $.0001$. We compared the performances of the standard PAC-Bayesian training algorithm (S), where the misclassification loss is replaced by a bounded version of the cross-entropy loss as in \cite{perezortiz2021tighter}, and the Cond-Gauss algorithm (G). We used the following four training objectives from \eqref{eq:bounds}:
\begin{align*}
    &\mathtt{invKL}: &&\kl^{-1}(\Err_S(\Q)|{\Pen}_\kappa)\,;\\
    &\mathtt{McAll}: &&\Err_S(\Q) + \sqrt{{\Pen}_\kappa/2}\,;\\
    &\mathtt{quad}: &&(\sqrt{\Err_S(\Q)+ {\Pen}_\kappa/2} + \sqrt{{\Pen}_\kappa/2})^2\,;\\
    &\mathtt{lbd}: &&\tfrac{1}{1-\lambda/2}(\Err_S(\Q) +  {\Pen}_\kappa/\lambda)\,,
\end{align*}
where the $\KL$ penalty is defined as 
\begin{equation}
    {\Pen}_\kappa = \frac{\kappa}{m}\left(\KL(\Q\|\PR) + \log\frac{2\sqrt m}{\delta}\right)\,.\tag{\ref{eq:def_Pen}}
\end{equation}
The factor $\kappa$ in \eqref{eq:def_Pen} can increase or reduce the weight of the $\KL$ term during the training. We experimented three different values for this parameter: $0.5$, $1$, and $2$. For the last objective, $\mathtt{lbd}$, the parameter $\lambda$ takes values in $(0,1)$ and is optimised during training\footnote{In our experiments, we initialised $\lambda$ at $0.5$ and then doubled the number of epochs, alternating one epoch of $\lambda$'s optimisation with one of optimisation for $\emme$ and $\esse$.}.

For all the different training settings,  the network was trained via SGD with momentum for 250 epochs with a learning rate $\eta=.005$ followed by 50 epochs with $\eta=.0001$. We tried using different values for the momentum: $0.5$, $0.7$, and $0.9$.  During the training, at the end of each epoch, we kept track of the bound \eqref{eq:bound}'s empirical value in order to pick the best epoch at the end of the training. 

Figure \ref{fig:MNIST_rnd} and Table \ref{tab:mnist_random} in the main text report our results.
\subsubsection{Data-dependent priors}
For the data-dependent priors, we used 50\% of the dataset to train $\PR$ and the remaining 50\% to train $\Q$. We always used the Cond-Gaussian algorithm for both prior and posterior. All the posteriors were trained with the $\mathtt{invKL}$ objective and $\kappa=1$, whilst for the prior, we experimented with both $\mathtt{invKL}$ (with $\kappa=0.1$) and with direct empirical risk minimisation ($\mathtt{ERM}$), meaning that the objective was simply $\Err_S(\Q)$. The initial prior's variances were set at $0.01$, while the means were randomly initialised (via the default PyTorch initialisation for each layer). We used different dropout values, as shown in Table \ref{tab:mnist_learnt}. The prior's training consisted of 750 epochs with $\eta=.005$, followed by 250 epochs with $\eta = .0001$, the posterior's training of 750 epochs with $\eta=10^{-5}$, followed by 250 epochs with $\eta=10^{-6}$. We used SGD with a momentum of $0.9$ for both priors and posteriors. The results of the experiment can be found in Table \ref{tab:mnist_learnt}.  
\begin{table}[ht!]
\centering
\caption{MNIST - Prior learnt on 50\% of the dataset}
\label{tab:mnist_learnt}
\begin{threeparttable}
\resizebox{\textwidth}{!}{\begin{tabular}{lccc|ccc||l|cccc|c}
  \toprule[1pt]
  \multicolumn{7}{c||}{\textbf{Prior}} & \multicolumn{6}{c}{\textbf{Posterior}}\\
  \midrule
  tm\tnote{a} & do\tnote{b} & pf\tnote{c} & iv\tnote{d} & l1\tnote{e} & l2\tnote{f} & p\tnote{g} & tm\tnote{a} & l1\tnote{e} & l2\tnote{f} & t\tnote{h} & p\tnote{g} & b\tnote{i}\\
  \midrule
  G $\mathtt{ERM}$ & 0 & - & .001 & .0010 & .0126 & 3.179 & G $\mathtt{invKL}$ & .0010 & .0122 & .0122\textsubscript{\textpm.0006} & 3.671 & .0164\\
  G $\mathtt{invKL}$ & 0 & .01 & .001 & .0008 & .0125 & 3.179 & G $\mathtt{invKL}$ & .0008 & .0119 & .0115\textsubscript{\textpm.0007} & 3.882 & .0162\\
  G $\mathtt{ERM}$ & .1 & - & .001 & .0010 & .0111 & 3.179 & G $\mathtt{invKL}$ & .0010 & .0107 & .0110\textsubscript{\textpm.0006} & 3.688 & .0148\\
  G $\mathtt{invKL}$ & .1 & .01 & .001 & .0006 & .0113 & 3.179 & G $\mathtt{invKL}$ & .0006 & .0107 & .0109\textsubscript{\textpm.0006} & 3.944 & .0149\\
  G $\mathtt{ERM}$ & .2 & - & .001 & .0011 & .0111 & 3.179 & G $\mathtt{invKL}$ & .0011 & .0107 & .0101\textsubscript{\textpm.0005} & 3.742 & .0148\\
  G $\mathtt{invKL}$ & .2 & .01 & .001 & .0010 & .0108 & 3.179 & G $\mathtt{invKL}$ & .0010 & .0104 & .0101\textsubscript{\textpm.0006} & 3.801 & \textbf{.0144}\\
  \bottomrule[1pt]
\end{tabular}
}
\begin{tablenotes}
\footnotesize{
\setlength{\columnsep}{0cm}
\setlength{\multicolsep}{0cm}
  \begin{multicols}{2}
    \item[a] tm: Training method.
    \item[b] do: Dropout probability for the prior's training.
    \item[c] pf: Penalty factor $\kappa$ for the prior's training objective. 
    \item[d] iv: Initial value of the prior's variances.
    \item[e] l1: Empirical error estimate on the prior dataset.
    \item[f] l2: Empirical error estimate on the posterior dataset.
    \item[g] p: $\KL$ penalty $\Pen$ \eqref{eq:defpen} in $10^{-4}$ units.
    \item[h] t: Test error\textsubscript{\textpm standard deviation} (from 1000 realisations).
    \item[i] b: Final PAC-Bayesian bound.
  \end{multicols}}
\end{tablenotes}
\end{threeparttable}
\end{table}
\subsection{CIFAR10}
As we had done for the MNIST dataset, for CIFAR10 we used only the standard training dataset (50000 labelled images). We trained a 9-layer architecture (6 convolutional + 3 linear layers) and a 15-layer architecture (12 convolutional + 3 linear layers). We experimented with data-dependent priors only, training $\PR$ with 50\% of the data for the 9-layer classifier and both with 50\% and 70\% for the 15-layer one. 
\subsubsection{9-layer architecture}
The 9-layer architecture had the following structure:
$$x\mapsto L_3\circ\phi\circ L_2\circ\phi\circ L_1\circ\phi\circ f_2\circ C_6\circ\phi\circ C_5\circ\phi\circ f_1\circ C_4\circ\phi\circ C_3\circ\phi\circ f_1\circ C_2\circ\phi\circ C_1(x)\,.$$
Here are detailed the different layers:
\begin{itemize}[noitemsep,topsep=0pt]
    \item $C_1$: convolutional layer; channels: IN 3, OUT 32; kernel: (3, 3); stride: (1, 1); padding(1, 1);
    \item $C_2$: convolutional layer; channels: IN 32, OUT 64; kernel: (3, 3); stride: (1, 1); padding(1, 1);
    \item $C_3$: convolutional layer; channels: IN 64, OUT 128; kernel: (3, 3); stride: (1, 1); padding(1, 1);
    \item $C_4$: convolutional layer; channels: IN 128, OUT 128; kernel: (3, 3); stride: (1, 1); padding(1, 1);
    \item $C_5$: convolutional layer; channels: IN 128, OUT 256; kernel: (3, 3); stride: (1, 1); padding(1, 1);
    \item $C_6$: convolutional layer; channels: IN 256, OUT 256; kernel: (3, 3); stride: (1, 1); padding(1, 1);
    \item $L_1$: linear layer; dimensions: IN 4096, OUT 1024;
    \item $L_2$: linear layer; dimensions: IN 1024, OUT 512;
    \item $L_3$: linear layer; dimensions: IN 512, OUT 10;
    \item $f_1$: max pool (kernel size = 2, stride = 2);
    \item $f_2$: max pool (kernel size = 2, stride = 2) \& flatten;
    \item $\phi$: ReLU activation component-wise.
\end{itemize}
All convolutional and linear layers are with bias.

The results for the 9-layer architecture are reported in Table \ref{tab:cifar_9_learnt} in the main text. After some preliminary experiments, we chose to train both priors and posteriors via the Cond-Gauss algorithm with the $\mathtt{invKL}$ objective. We used a small factor $\kappa$ for the prior, to avoid regularising too much, whilst $\kappa$ was $1$ for the posterior. We tried different values for the dropout and $\kappa$ in the prior's training (see Table \ref{tab:cifar_9_learnt}). We used SGD with momentum $0.9$ for both prior and posterior. For $\PR$ the training consisted of $1500$ epochs with $\eta=.005$ followed by $500$ epochs with $\eta=.0001$, whilst $\Q$ was trained for 1500 epochs with $\eta=10^{-5}$, plus 500 epochs with $\eta=10^{-6}$.
\subsubsection{15-layer architecture}
The 15-layer architecture had the following structure:
\begin{align*}
    x\mapsto &L_3\circ\phi\circ L_2\circ\phi\circ L_1\circ\phi\circ f_2\circ C_{12}\circ\phi\circ C_{11}\circ\phi\circ C_{10}\circ\phi\circ C_{9}\circ\phi\circ f_1\circ C_8\circ\phi\\
    &\circ C_7\circ\phi\circ f_2\circ C_6\circ\phi\circ C_5\circ\phi\circ f_1\circ C_4\circ\phi\circ C_3\circ\phi\circ f_1\circ C_2\circ\phi\circ C_1(x)\,.
\end{align*}
Here are detailed the different layers:
\begin{table}[th!]
\centering
\caption{CIFAR10 - 15 layers - Prior learnt on 50\% and 70\% of the dataset}
\label{tab:cifar_15_learnt}
\begin{threeparttable}
\resizebox{\textwidth}{!}{\begin{tabular}{clccc|ccc||l|cccc|c}
  \cmidrule[1pt]{2-14}
  & \multicolumn{7}{c||}{\textbf{Prior}} & \multicolumn{6}{c}{\textbf{Posterior}}\\
  \cmidrule{2-14}
  & tm\tnote{a} & do\tnote{b} & pf\tnote{c} & iv\tnote{d} & l1\tnote{e} & l2\tnote{f} & p\tnote{g} & tm\tnote{a} & l1\tnote{e} & l2\tnote{f} & t\tnote{h} & p\tnote{g} & b\tnote{i}\\
  \cmidrule{2-14}
  & \multicolumn{13}{c}{\textbf{Prior trained on 50\% of the dataset}}\\
  \cmidrule{2-14}
  \multirow{4}{*}{\rotatebox{90}{\begin{tabular}{@{}c@{}}Pre-Train\\do=.1 \end{tabular}}}&  G $\mathtt{ERM}$ & 0 & - & .001 & .0090 & .1946 & 3.778 & G $\mathtt{invKL}$ & .0090 & .1924 & .1933\textsubscript{\textpm.0020} & 4.775 & .2082\\
  & G $\mathtt{invKL}$ & 0 & .01 & .001 & .0085 & .1937 & 3.778 & G $\mathtt{invKL}$ & .0084 & .1909 & .1922\textsubscript{\textpm.0022} & 4.913 & .2068\\
  & G $\mathtt{ERM}$ & .1 & - & .001 & .0139 & .1722 & 3.778 & G $\mathtt{invKL}$ & .0139 & .1709 & .1736\textsubscript{\textpm.0018} & 4.386 & \textbf{.1855}\\
  & G $\mathtt{invKL}$ & .1 & .01 & .001 & .0222 & .1746 & 3.778 & G $\mathtt{invKL}$ & .0222 & .1725 & .1760\textsubscript{\textpm.0020} & 4.703 & .1875\\
  \cmidrule{2-14}
  \multirow{4}{*}{\rotatebox{90}{\begin{tabular}{@{}c@{}}Pre-Train\\do=.2\end{tabular}}} & G $\mathtt{ERM}$ & 0 & - & .001 & .0214 & .1996 & 3.778 & G $\mathtt{invKL}$ & .0214 & .1974 & .1939\textsubscript{\textpm.0020} & 4.734 & .2133\\
  & G $\mathtt{invKL}$ & 0 & .01 & .001 & .0169 & .1963 & 3.778 & G $\mathtt{invKL}$ & .0169 & .1941 & .1930\textsubscript{\textpm.0022} & 4.859 & .2100\\
  & G $\mathtt{ERM}$ & .1 & - & .001 & .0240 & .1772 & 3.778 & G $\mathtt{invKL}$ & .0240 & .1758 & .1791\textsubscript{\textpm.0017} & 4.474 & .1907\\
  & G $\mathtt{invKL}$ & .1 & .01 & .001 & .0394 & .1764 & 3.778 & G $\mathtt{invKL}$ & .0393 & .1747 & .1734\textsubscript{\textpm.0019} & 4.606 & .1897\\
  \cmidrule{2-14}
  & \multicolumn{13}{c}{\textbf{Prior trained on 70\% of the dataset}}\\
  \cmidrule{2-14}
  \multirow{4}{*}{\rotatebox{90}{\begin{tabular}{@{}c@{}}Pre-Train\\do=.1 \end{tabular}}}& G $\mathtt{ERM}$ & 0 & - & .001 & .0057 & .1616 & 6.127 & G $\mathtt{invKL}$ & .0057 & .1602 & .1643\textsubscript{\textpm.0020} & 6.882 & .1774\\
  & G $\mathtt{invKL}$ & 0 & .01 & .001 & .0062 & .1634 & 6.127 & G $\mathtt{invKL}$ & .0062 & .1617 & .1648\textsubscript{\textpm.0021} & 7.203 & .1793\\
  & G $\mathtt{ERM}$ & .1 & - & .001 & .0098 & .1443 & 6.127 & G $\mathtt{invKL}$ & .0098 & .1430 & .1470\textsubscript{\textpm.0017} & 7.006 & \textbf{.1595}\\
  & G $\mathtt{invKL}$ & .1 & .01 & .001 & .0180 & .1467 & 6.127 & G $\mathtt{invKL}$ & .0178 & .1446 & .1506\textsubscript{\textpm.0019} & 7.374 & .1616\\
  \cmidrule{2-14}
  \multirow{4}{*}{\rotatebox{90}{\begin{tabular}{@{}c@{}}Pre-Train\\do=.2\end{tabular}}} & G $\mathtt{ERM}$ & 0 & - & .001 & .0151 & .1639 & 6.127 & G $\mathtt{invKL}$ & .0151 & .1622 & .1696\textsubscript{\textpm.0018} & 7.161 & .1797\\
  & G $\mathtt{invKL}$ & 0 & .01 & .001 & .0127 & .1629 & 6.127 & G $\mathtt{invKL}$ & .0127 & .1611 & .1656\textsubscript{\textpm.0020} & 7.293 & .1787\\
  & G $\mathtt{ERM}$ & .1 & - & .001 & .0175 & .1484 & 6.127 & G $\mathtt{invKL}$ & .0175 & .1471 & .1506\textsubscript{\textpm.0016} & 7.043 & .1638\\
  & G $\mathtt{invKL}$ & .1 & .01 & .001 & .0306 & .1500 & 6.127 & G $\mathtt{invKL}$ & .0305 & .1484 & .1498\textsubscript{\textpm.0018} & 7.090 & .1652\\
  \cmidrule[1pt]{2-14}
\end{tabular}}
\begin{tablenotes}
\footnotesize{
\setlength{\columnsep}{-1cm}
\setlength{\multicolsep}{0cm}
  \begin{multicols}{2}
    \item[a] tm: Training method.
    \item[b] do: Dropout probability for the prior's training.
    \item[c] pf: Penalty factor $\kappa$ for the prior's training objective. 
    \item[d] iv: Initial value of the prior's variances.
    \item[e] l1: Empirical error estimate on the prior dataset.
    \item[f] l2: Empirical error estimate on the posterior dataset.
    \item[g] p: $\KL$ penalty $\Pen$ \eqref{eq:defpen} in $10^{-4}$ units.
    \item[h] t: Test error\textsubscript{\textpm standard deviation} (from 1000 realisations).
    \item[i] b: Final PAC-Bayesian bound.
  \end{multicols}}
\end{tablenotes}
\end{threeparttable}
\end{table}
\begin{itemize}[noitemsep,topsep=0pt]
    \item $C_1$: convolutional layer; channels: IN 3, OUT 32; kernel: (3, 3); stride: (1, 1); padding(1, 1);
    \item $C_2$: convolutional layer; channels: IN 32, OUT 64; kernel: (3, 3); stride: (1, 1); padding(1, 1);
    \item $C_3$: convolutional layer; channels: IN 64, OUT 128; kernel: (3, 3); stride: (1, 1); padding(1, 1);
    \item $C_4$: convolutional layer; channels: IN 128, OUT 128; kernel: (3, 3); stride: (1, 1); padding(1, 1);
    \item $C_5$: convolutional layer; channels: IN 128, OUT 256; kernel: (3, 3); stride: (1, 1); padding(1, 1);
    \item $C_6$: convolutional layer; channels: IN 256, OUT 256; kernel: (3, 3); stride: (1, 1); padding(1, 1);
    \item $C_7$: convolutional layer; channels: IN 256, OUT 256; kernel: (3, 3); stride: (1, 1); padding(1, 1);
    \item $C_8$: convolutional layer; channels: IN 256, OUT 256; kernel: (3, 3); stride: (1, 1); padding(1, 1);
    \item $C_9$: convolutional layer; channels: IN 256, OUT 512; kernel: (3, 3); stride: (1, 1); padding(1, 1);
    \item $C_{10}$: convolutional layer; channels: IN 512, OUT 512; kernel: (3, 3); stride: (1, 1); padding(1, 1);
    \item $C_{11}$: convolutional layer; channels: IN 512, OUT 512; kernel: (3, 3); stride: (1, 1); padding(1, 1);
    \item $C_{12}$: convolutional layer; channels: IN 512, OUT 512; kernel: (3, 3); stride: (1, 1); padding(1, 1);
    \item $L_1$: linear layer; dimensions: IN 2048, OUT 1024;
    \item $L_2$: linear layer; dimensions: IN 1024, OUT 512;
    \item $L_3$: linear layer; dimensions: IN 512, OUT 10;
    \item $f_1$: max pool (kernel size = 2, stride = 2);
    \item $f_2$: max pool (kernel size = 2, stride = 2) \& flatten;
    \item $\phi$: ReLU activation component-wise.
\end{itemize}
All convolutional and linear layers are with bias.

For the 15-layer architecture, we experimented different prior trainings, with 50\% and 70\% of the training dataset. In both cases, it was necessary to introduce an initial pre-training for the prior's means, as otherwise the Cond-Gauss algorithm alone could not significantly decrease the training objective. First, we initialised the means with an orthogonal initialisation, as suggested in \cite{huorthogonal}. Then we optimised them by training a deterministic network (with the same architecture) using the cross-entropy loss on the prior's dataset, for $50$ epochs with $\eta=.005$. Finally, via the Cond-Gauss algorithm, we completed the prior's training and proceeded with the posterior's tuning   following the same learning rate schedule as for the 9-layer case. We always used SGD with momentum $0.9$. Different objectives and dropout factors were used for training the prior, as detailed in Table \ref{tab:cifar_15_learnt}, which also reports the results of our experiment.
\end{document}